\documentclass[sigconf]{acmart}

\usepackage{balance}
\usepackage[dvipsnames]{xcolor}
\usepackage{multirow}
\usepackage{graphicx}
\usepackage{subcaption}
\usepackage{enumitem}

\usepackage{amsmath,amssymb,amsfonts}
\usepackage{algorithm}
\usepackage{algorithmic}

\usepackage{tabularx}

\usepackage{diagbox}

\usepackage[english]{babel}
\usepackage[T1]{fontenc}

\usepackage{amsmath,amsthm,amssymb,bbm}
\allowdisplaybreaks[4]

\usepackage{enumitem} 

\usepackage{amsfonts} 
\usepackage{stmaryrd} 
\SetSymbolFont{stmry}{bold}{U}{stmry}{m}{n} 
\usepackage{bm} 

\RequirePackage[l2tabu, orthodox]{nag} 
\usepackage{booktabs} 
\usepackage{color} 
\usepackage{framed} 
\usepackage{graphicx} 
\usepackage{hyperref} 
\usepackage{lipsum} 
\usepackage{listings} 
\usepackage{tikz} 
\usetikzlibrary{positioning,chains,fit,shapes,calc}
\usepackage[colorinlistoftodos]{todonotes} 
\usepackage{url} 
\usepackage{xspace} 
\usepackage{cleveref} 
\crefname{section}{Sec.}{Sec.}
\Crefname{section}{Sec.}{Sec.}

\newcommand{\orcidID}[1]{\href{https://orcid.org/#1}{\includegraphics[height=0.8em]{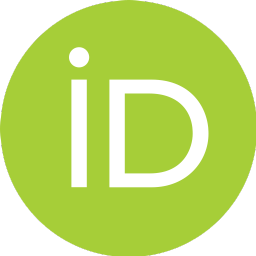}}}

\newcolumntype{C}[1]{>{\centering\arraybackslash\hspace{0pt}}p{#1}}
\def\BibTeX{{\rm B\kern-.05em{\sc i\kern-.025em b}\kern-.08em
   T\kern-.1667em\lower.7ex\hbox{E}\kern-.125emX}}

\newif\ifsubmission
\submissiontrue

\theoremstyle{plain}

\newtheorem{theorem}{Theorem}
\newtheorem*{thm*}{Theorem}

\newtheorem*{prop*}{Proposition}

\newtheorem*{lem*}{Lemma}

\theoremstyle{definition}

\theoremstyle{plain}

\newtheorem*{conve*}{Convention}
\theoremstyle{remark}

\newtheorem*{rmk*}{Remark}

\newtheorem*{notation*}{Notation}

\let\originalleft\left
\let\originalright\right
\renewcommand{\left}{\mathopen{}\mathclose\bgroup\originalleft}
\renewcommand{\right}{\aftergroup\egroup\originalright}

\newcommand*{\mbb}{\mathbb}

\newcommand*{\Z}{\ensuremath{\mbb{Z}}}

\newcommand*{\numClients}{\ensuremath{N}}
\newcommand*{\numMal}{\ensuremath{M}}
\newcommand*{\dimModel}{\ensuremath{d}}
\newcommand*{\model}{\alg{w}}
\newcommand*{\aggModel}{\bar{\model}}
\newcommand*{\cosSim}{\ensuremath{\mathrm{cos\_sim}}}

\newcommand*{\precision}{\ensuremath{\lambda}}

\newcommand*{\rand}{\alg{o}}
\newcommand*{\com}{\alg{c}}
\newcommand*{\oAvg}{\bar{\rand}} 
\newcommand*{\cAvg}{\bar{\com}} 
\newcommand*{\ip}{\alg{ip}} 
\newcommand*{\THD}{\mathrm{TH}}
\newcommand*{\drand}{\Delta}
\newcommand*{\agg}{\alg{agg}}

\newcommand*{\comEll}{\com^{\ell}} 
\newcommand*{\comNorm}{\com^{\mathrm{Norm}}} 
\newcommand*{\comNormSq}{\com^{\mathrm{Norm}^2}} 
\newcommand*{\comCard}{\com^{\alg{card}}} 

\newcommand*{\NIZK}{{\Pi}}

\newcommand*{\ZKP}{\mathsf{NIZK}}
\newcommand*{\COM}{\alg{COM}}
\AtBeginDocument{}

\newcommand*{\msg}{\mathsf{msg}}

\newcommand*{\GG}{\mathbb{G}}

\newcommand*{\alg}[1]{\mathsf{#1}}

\AtBeginDocument{%
  \providecommand\BibTeX{{%
    Bib\TeX}}%
}

\setcopyright{acmlicensed}
\copyrightyear{2026}
\acmYear{2026}
\acmDOI{XX.XXXX/XXXXXXX.XXXXXXX}

\acmConference[CONFERENCE '27]
  {ACM ASIA Conference on Computer and Communications Security}
  {12-16 July 2027}
  {Macau}

\acmBooktitle{Proceedings of the Full Name of the ACM Conference
  (CONFERENCE '26), Month DD--DD, 2026, City, Country}

\acmISBN{978-1-4503-XXXX-X/2026/XX}

\begin{document}

\title{End-to-End Verifiable and Robust Federated Learning}

%

\author{Doryan Lesaignoux}
\affiliation{%
  \institution{AIT Austrian Institute of Technology}
  \city{Vienna}
  \country{Austria}
}

\author{Enrique Mármol Campos}
\affiliation{%
  \department{Department of Computer Engineering}
  \institution{University of Murcia}
  \city{Murcia}
  \country{Spain}
}

\author{Gabriele Spini}
\affiliation{%
  \institution{AIT Austrian Institute of Technology}
  \city{Vienna}
  \country{Austria}
}

\author{José L. Hernández-Ramos}
\affiliation{%
  \department{Department of Computer Engineering}
  \institution{University of Murcia}
  \city{Murcia}
  \country{Spain}
}

\author{Stephan Krenn}
\affiliation{%
  \institution{AIT Austrian Institute of Technology}
  \city{Vienna}
  \country{Austria}
}

\renewcommand{\shortauthors}{Lesaignoux et al.}


\begin{abstract}
Federated learning enables multiple parties to train a shared model
without centralizing raw data with the help of an aggregator, but
introduces integrity risks once participants or infrastructure are not
fully trustworthy. Two requirements are particularly important:
robustness to poisoned or Byzantine client updates, and verifiability
of the aggregator so that clients or third parties can audit the
reported aggregation without learning individual updates. Existing work
has largely treated these goals separately, and efficient public
verifiability for robust, outlier-excluding aggregation remains
limited. We present a verifiable federated learning protocol that makes a robust
aggregation pipeline publicly auditable. Our design combines
cryptographic commitments with non-interactive zero-knowledge proofs
to certify both (i) cosine-similarity-based outlier exclusion and
(ii) aggregation over the selected set, without revealing individual
client updates to verifiers. In experiments under representative poisoning attacks, our method maintains high accuracy, with an average accuracy loss below 4\%
across the evaluated configurations, while keeping verification
overhead practical: proof artifacts can be generated and verified
within minutes at the scale studied. In summary, our results show that
robust outlier exclusion and public verifiability can be jointly
achieved in a federated learning setting.
\end{abstract}

%



\maketitle


\section{Introduction}\label{sec:introduction}

Federated Learning (FL) enables multiple data holders to train a shared machine-learning model without centralizing raw data~\cite{mcmahan2017communication,zhang2021survey}. In a standard FL workflow, clients train local models on private datasets and send model updates to a central server (the aggregator), which combines them using an aggregation function. This design provides a practical path to collaborative learning under privacy or confidentiality constraints.

Despite its privacy benefits, FL raises major integrity challenges once participants and infrastructure are not fully trusted. Trustworthiness in FL requires handling at least the following three complementary threats:

\begin{enumerate}[label=(\roman*)]
    \item \label{threat:byzantine}
    malicious or Byzantine clients that submit poisoned updates to
    degrade convergence or implant backdoors%
    ~\cite{bhagoji2019analyzing,shi2022challenges};

    \item \label{threat:privacy}
    the aggregator should learn as little information as possible about
    the participants' local data or model updates, ideally revealing
    only the final aggregated result while keeping individual
    contributions private%
    ~\cite{behnia2024efficient,DBLP:journals/iacr/ChandranOS26,lu25}; and

    \item \label{threat:aggregator}
    a deviating aggregator that may tamper with the aggregation
    computation or its outcome%
    ~\cite{zhu2024malicious}.
\end{enumerate}

These threats are often addressed separately, through robust aggregation tolerating outliers and poisoned updates~\cite{uddin2025systematic,yin2018byzantine,blanchard2017machine,pillutla2022robust,fung2018mitigating, li2021lomar, mu2024feddmc,campos2025flaegis}, deployment of cryptographic schemes such as secure multi-party computation or fully homomorphic encryption~\cite{lu25}, and verifiability mechanisms that allow clients (or external auditors) to check that aggregation was performed correctly~\cite{xu2025efficient,sabater2022accurate, zhu2024malicious} without requiring plaintext access to the individual updates.

However, considering these challenges independently leads to incompatibilities, as the solutions can, in general, not be combined in a ``plug-and-play'' fashion.
This is not only due to incompatible parameters or cryptographic methods in~\ref{threat:privacy} and~\ref{threat:aggregator}, but more fundamentally also to the internal structure of the aggregation techniques in~\ref{threat:byzantine}.
Specifically, many robustness mechanisms rely on highly non-linear or data-dependent steps (including but not limited to aspects like sorting, iterative thresholding, or clustering decisions), which are hard to realize in a data-oblivious way.
For instance, no data-dependent branching may be recognizable to the aggregator (in case of~\ref{threat:privacy}) or a verifier (in case of~\ref{threat:aggregator}), making simple functions like computing medians or sorting computationally expensive;
even worse, already comparing two values becomes relatively expensive, as data-dependent computations must often be implemented with worst-case execution behavior (e.g., in the bit-length of the values to be compared) to not leak any information about the underlying data (e.g., at which bit position the decision was taken).
Finally, despite recent progress, computing on floating-point numbers introduces significant overheads, suggesting integer computations as the natural candidate.
This leads to a conflict between high robustness guarantees on the one hand, and high privacy and verifiability needs on the other.

In this paper we thus aim at solutions simultaneously achieving~\ref{threat:byzantine} and~\ref{threat:aggregator}, leading to the following research question:

\medskip

\noindent\colorbox{gray!25}{%
  \parbox{0.98\columnwidth}{%
    Can a verifiable robust aggregation pipeline be designed that certifies both outlier selection and final aggregation while retaining competitive empirical robustness against representative poisoning attacks within practically plausible runtime?
  }%
}

\medskip

It is worth noting that we deliberately leave aside threat~\ref{threat:privacy}, as addressing all three challenges simultaneously would be beyond the scope of a single article.
Nevertheless, the challenges in achieving robustness in~\ref{threat:privacy} and~\ref{threat:aggregator} are often closely related, particularly regarding data-obliviousness and efficient computation over protected values, so that we expect some of the techniques developed here may carry over.

\subsection{Contributions and Overview of results}
In particular, we make the following contributions:

\paragraph{ZK-friendly robust aggregation}
We propose a threshold-based robust aggregation pipeline designed around operations suitable for zero-knowledge proofs, mainly algebraic relations and comparisons. Unlike conventional verifiable aggregation, which typically focuses on secure sums, our design balances robustness and proof efficiency. We use cosine similarity~\cite{KOTU2019343} to detect inconsistent updates, exclude them through threshold-based selection, and aggregate the remaining updates using weighted aggregation. While the basic version uses a static threshold, we also consider semi-static and dynamic variants. Against several poisoning attacks, our method achieves low and stable accuracy degradation while using a proof-friendly aggregation rule, avoiding catastrophic failures across the evaluated configurations. 

\paragraph{Zero-knowledge public verifiability}
We introduce a commitment- and NIZK-based verification layer that enables public auditing of the complete robust aggregation pipeline. Using commitments published by clients, the aggregator proves: (i) correct cosine-similarity computations, (ii) correct threshold-based inclusion or exclusion, and (iii) correct weighted aggregation over the selected updates. These proofs reveal neither individual client updates nor selection decisions, including the number of rejected clients. Hiding this information also limits feedback that adaptive attackers could exploit to craft updates close to the rejection threshold.

\paragraph{Characterizing robustness--verifiability costs}
We evaluate both robustness against representative poisoning attacks and the cost of public verification in terms of proof size, generation time, and verification time, including scalability with model dimension and number of clients. We implement a prototype of our verification pipeline and analyze the impact of alternative configurations, including metadata disclosure and more advanced thresholding strategies.
\subsection{Related Work}\label{sec:Related}

FL enables collaborative training without centralizing raw data~\cite{mcmahan2017communication}. In adversarial settings, integrity is mainly addressed through robustness against malicious clients and verifiability of the aggregator. However, most verifiable FL schemes remain limited to largely linear aggregation rules~\cite{KR25}.

\paragraph{Robustness to poisoning and Byzantine clients}
Malicious clients can degrade convergence or implant backdoors through crafted updates~\cite{bhagoji2019analyzing,fang2020local}. Representative attacks include LIE~\cite{baruch2019little} and Min-Max/Min-Sum~\cite{shejwalkar2021manipulating}. Defenses include median/trimmed mean~\cite{yin2018byzantine}, Krum/Multi-Krum~\cite{blanchard2017machine}, Bulyan~\cite{mhamdi2018hidden}, geometric-median approaches~\cite{pillutla2022robust}, trust-based methods~\cite{cao2020fltrust}, and similarity-based detectors~\cite{fung2018mitigating}. Yet, all these assume an honest execution of the aggregation procedure.

\paragraph{Verifiable aggregation}
Verifiable FL aims to prove that the server correctly computes the aggregate. VerifyNet~\cite{xu2019verifynet} combines privacy masking with homomorphic-hash verification, while VeriFL~\cite{guo2020v} reduces verification overhead for high-dimensional gradients. Other approaches use HE~\cite{wang2024priverifl}, blockchain-based auditing~\cite{wang2025privacy,chen2024litechain}, or commitment- and ZK-based verification~\cite{sabater2022accurate,bottoni2022verifiable}. These techniques are most efficient for linear computations, whereas data-dependent filtering and selection are considerably harder to prove~\cite{KR25}.

\paragraph{Combining robustness and verifiability}
Few works jointly consider malicious clients and a deviating aggregator. BVDFed~\cite{gao2024bvdfed} combines loss-based filtering with aggregate verification, RiseFL~\cite{zhu2024secure} verifies per-update constraints using ZK proofs, and RFLPV~\cite{wang2024rflpv} combines robustness with verifiable aggregation. Yet, these approaches do not verify the robustness decisions themselves.

In contrast, we target end-to-end verification of an outlier-excluding aggregation pipeline under both malicious clients and a potentially malicious aggregator. Our protocol uses proof-friendly linear relations and comparisons to verify similarity computation, threshold-based selection, and aggregation over accepted updates, while preserving update privacy.

\paragraph{Combining secure and verifiable aggregation}
Secure aggregation hides individual updates while allowing computation of their aggregate. Existing schemes commonly use MPC or secret sharing~\cite{buyukates2024lightverifl,roy2022eiffel,behnia2024efficient,peng2023communication} or HE~\cite{madi2021secure,wang2025effective}, often combined with verification mechanisms. This aspect is outside our scope, since we assume plaintext client inputs. Moreover, these works usually consider a malicious aggregator but standard aggregation, whereas our model also includes malicious clients.

\paragraph{Combining secure and robustness}
Combining privacy and Byzantine robustness is difficult because robust rules require comparisons and distance computations that are expensive under HE or MPC. BREA and ByzSecAgg~\cite{so2020byzantine,jahani2025byzsecagg} combine MPC with Multi-Krum-style techniques, ShieldFL~\cite{ma2022shieldfl} uses HE and cosine similarity, and SEAR~\cite{zhao2021sear} relies on a TEE. PRFL~\cite{liu2024prfl} combines MPC with FLTrust under a malicious-server model, but focuses mainly on privacy leakage through deviation or collusion. Overall, existing secure-robust schemes rarely address an aggregator that actively manipulates the robust aggregation result.

\subsection{Outline}
\Cref{sec:Preliminaries} introduces the required concepts, while \Cref{sec:Methodology} presents the system and threat model. 
\Cref{sec:robust_verif_fl} describes the proposed method, which is evaluated in \Cref{sec:Results}.
We briefly conclude in \Cref{sec:Conclusions}.
\section{Preliminaries}\label{sec:Preliminaries}

In this section, we overview poisoning attacks in FL, summarize defenses based on similarity filtering and robust aggregation, and outline the cryptographic building blocks used in our framework.
\Cref{table:parameters} summarizes the notation used throughout the article.

\begin{table}[th!]
	\begin{center}
\caption{Overview of main notation and parameters}
\label{table:parameters}
\begin{tabular}{cl}
\toprule
\textbf{Parameter} & \textbf{Description} \\
\midrule
$G,H$ & Generators of $\GG$ \\
$\numClients$ & Number of clients \\
$\numMal$ & Number of malicious clients \\
$\model_i$ & Weights of $i^{th}$ client \\
$\dimModel$ & Dimension of the weights \\
$\aggModel$ & Aggregated weights \\
$\model$ & Final aggregation of weights \\
$q$ & Arithmetic modulus \\
$\precision$ & Precision parameter \\
$\THD$ & Threshold of cosine similarities \\
$\com, \rand$ &Commitment and opening value\\
 $\ell_i$ & Original norm of weight $\model_i$\\
\bottomrule
\end{tabular}
\end{center}
\end{table}

\subsection{Poisoning Attacks in FL and Mitigation}

Poisoning attacks in FL aim to degrade training or induce targeted misbehavior by manipulating the information contributed by a subset of clients, typically through corrupted local data or adversarially crafted model updates~\cite{xia2023poisoning}.
We distinguish two common views.
In data poisoning, the adversary modifies a client’s local training data to bias the learned update (e.g., to implant a backdoor trigger or to shift decision boundaries) \cite{neto2023survey,lyu2020threats}.
In model poisoning (or local update poisoning), the adversary directly crafts the transmitted update either to broadly hinder convergence or to achieve a targeted effect while remaining hard to detect~\cite{bhagoji2019analyzing,fang2020local}.
Modern evasion attacks explicitly optimize poisoned updates to pass common defenses, including small-perturbation strategies (e.g., LIE~\cite{baruch2019little}) and optimization-based attacks tailored to robust aggregation rules (e.g., Min-Max/Min-Sum~\cite{shejwalkar2021manipulating}).

Mitigations broadly fall into two families.
Robust aggregation replaces simple averaging with rules that reduce the influence of outliers, such as coordinate-wise robust statistics (median/trimmed mean) or distance-based selection (Krum)~\cite{yin2018byzantine, blanchard2017machine}.
Filtering-based defenses attempt to identify suspicious clients, often using similarity signals, and exclude them before aggregation~\cite{fung2018mitigating, li2021lomar, mu2024feddmc}.
Recent evidence supports two-stage pipelines that combine filtering with a robust aggregator as particularly effective against poisoning~\cite{campos2025flaegis}.
In this paper, we follow this design principle and use cosine similarity as the key signal for update consistency:

\begin{equation}\label{eq:cosimilarity}
    \mathrm{cos\_sim}(\mathbf{x},\mathbf{y})
= \frac{\left\langle\mathbf{x}, \mathbf{y}\right\rangle}{\|\mathbf{x}\|_2\,\|\mathbf{y}\|_2}
= \frac{\sum_{i=1}^{n} x_i\,y_i}{\sqrt{\sum_{i=1}^{n} x_i^2}\;\sqrt{\sum_{i=1}^{n} y_i^2}}
\end{equation}

which we leverage to separate consistent updates from outliers prior to aggregation~\cite{KOTU2019343}.
After filtering, we aggregate the retained updates using an averaging-style rule;
our full pipeline additionally supports similarity-aware weighting to reduce the impact of borderline or misclassified updates.

\subsection{Cryptographic Preliminaries}
This subsection introduces the cryptographic notions and proof systems.
The presentation of the underlying primitives remains at a high level; the security goals and resulting guarantees of the complete protocol are stated in \Cref{s:security_model}.
For further details, we refer the reader to the appendix and the original literature.

In the following, $(\GG, +)$ denotes a finite cyclic group and $\Z_q$ the finite field of integers modulo a prime $q$.

\paragraph{Commitments}
  A commitment scheme is a cryptographic protocol that allows one party to commit to a chosen value while keeping it hidden, with the ability to reveal it later.
  It guarantees hiding, so the value remains secret until opening, and binding, so the committer cannot change the value after committing.

  Our construction will rely on so-called Pedersen commitments~\cite{Pedersen1991}.
  Given two generators $G,H$ in $\GG$, a commitment to a message $\msg\in\Z_q$ is given by $\com = \COM(\msg,\rand) = \msg\cdot G + \rand\cdot H$ for a uniformly random opening $\rand$.
  Given a commitment $\com$, a message $\msg$, and an opening $\rand$, verification is done in the canonical way by checking that the above equality holds.

\paragraph{Zero-knowledge proofs of knowledge}
  These are cryptographic protocols that allow a prover to convince a verifier that they possess a secret value satisfying a given relation, without revealing any information about the secret itself other than what is already revealed by the claim itself.
  Such a protocol guarantees completeness, soundness, and zero-knowledge, meaning that an honest prover can convince an honest verifier, a cheating prover cannot convince the verifier without knowing the secret, and the verifier learns nothing other than the truth of the statement.
  A \emph{non-interactive zero-knowledge proof} (NIZK) is a variant in which the proof consists of a single message from the prover to the verifier, typically enabled by a common reference string or a random oracle model.

  We use the Camenisch-Stadler framework for representing proof goals~\cite{Camenisch1997}.
  For instance, 
\[
\begin{aligned}
\NIZK \gets \ZKP\big[(\alpha,\beta,\gamma):\;
& Y = \alpha G + \beta H \;\land \\
& Z = \alpha G + \gamma H \;\land\;\gamma = \alpha \cdot \beta
\big]
\end{aligned}
\]
denotes a NIZK demonstrating knowledge of witnesses $\alpha, \beta, \gamma$ satisfying the given relations.
All values except witnesses are assumed to be public.
All statements used in this paper can efficiently be proven using standard techniques~\cite{schnorr1991efficient,fiat-shamir,maurer-main-protocol,DBLP:conf/crypto/CramerDS94,DBLP:phd/dnb/Krenn12}.
\section{Threat Model}\label{sec:Methodology}

This section introduces the FL setting considered in this work, along with the assumptions that define our threat model.
We then describe the bulletin board for letting clients or external auditors verify that the aggregation was performed correctly.

\subsection{Threat Model}
We consider a standard cross-silo FL deployment with $\numClients$ clients training a shared model through a central aggregator.
In each training round, clients compute a local update and send it to the aggregator, which applies an aggregation procedure and returns the resulting global model.

\paragraph{Malicious clients}
We assume that up to $\numMal$ clients are Byzantine and may send arbitrary weights with the goal of degrading the global model's utility (untargeted poisoning) or inducing targeted behavior (e.g., backdoors).
Malicious clients may collude and coordinate their weights. Following common assumptions in Byzantine-robust FL, we consider $\numMal \leq \numClients/2$, which is also required by clustering/filtering-based defenses.

\paragraph{Deviating aggregator}
We also consider a corrupted aggregator that may deviate arbitrarily from the prescribed protocol.
Concretely, the aggregator may (i) output an aggregate that is inconsistent with the received client weights, (ii) alter intermediate computations (e.g., similarity scores, thresholds, or selection decisions), or (iii) selectively omit, re-weight, or replace weights.
This captures the integrity threat that verifiability is designed to address.

\paragraph{Visibility and Knowledge}
We assume the aggregator receives plaintext client weights in each round (secure aggregation is beyond the scope of this work).
Honest clients cannot observe other clients’ weights, while malicious clients can share information.
Attackers may know the learning task and training procedure;
the key threat is that they can craft weights (clients) or deviate from the prescribed computation (aggregator).

\paragraph{Objective}
The goal of the adversaries, Byzantine clients and/or a corrupted aggregator, is to cause honest clients to accept a global model whose performance is degraded or maliciously biased.

Our robustness mechanism mitigates the impact of Byzantine weights under its assumptions, while the verifiability layer ensures that the aggregator cannot deviate from the prescribed robust aggregation pipeline without being detected by auditors.

\subsection{Security Model and Guarantees}\label{s:security_model}
We consider a probabilistic polynomial-time adversary that may corrupt the aggregator and an arbitrary subset of clients, subject to the robustness assumptions stated above.

Our security goals are then as follows.

\paragraph{Aggregation soundness}
No probabilistic polynomial-time aggregator can produce an accepting proof for an output that is inconsistent with the publicly specified aggregation procedure applied to the complete finalized set of validly authenticated inputs for the corresponding training round, except with negligible probability.

\paragraph{Public-transcript privacy}
The public transcript, including the commitments and zero-knowledge proofs, reveals no information about individual client updates, client-level inclusion decisions, or the number of selected clients beyond what is implied by the public inputs, authenticated submission metadata, and the released aggregate.
This guarantee applies to public verifiers and to clients that do not collude with the aggregator.
The aggregator receives the client updates and commitment openings in plaintext and is therefore outside the privacy boundary.

\section{Verifiable and Robust Federated Learning}\label{sec:robust_verif_fl}

As stated before, FL is often deployed in environments where neither clients nor the coordinating infrastructure are fully trusted. Robust aggregation is thus needed to tolerate malicious participants, while verifiability enables clients or third parties to audit the aggregator. Yet, public verification constrains robust defenses: methods such as coordinate-wise median or Krum rely on sorting and data-dependent operations that are costly to prove at model scale~\cite{yin2018byzantine}. This motivates robust pipelines based on proof-friendly primitives, mainly linear relations and simple comparisons. In this section, we present our robust and publicly verifiable FL protocol. We first describe the aggregation pipeline and design choices that preserve robustness while enabling efficient verification, then introduce the corresponding mechanisms, where commitments and NIZKPs allow third parties to audit the computation without revealing client weights.

We build on FLAegis~\cite{campos2025flaegis}, which combines filtering of suspicious updates with robust aggregation to mitigate malicious updates that escape detection. However, several operations in FLAegis are difficult to verify efficiently. We therefore redesign the pipeline into a more ``ZK-friendly'' variant, replacing complex procedures with operations suitable for efficient proofs. The resulting filtering and aggregation steps can be publicly certified while retaining robustness against poisoning attacks. More expressive variants are possible at the cost of higher verification complexity; their trade-offs are discussed in the following subsections and the appendix.

\Cref{alg:main_algorithm} specifies the procedure of our method, consisting of three main stages: (1) compute a similarity score for each client update, (2) cluster division, and (3) aggregate the selected weights.

\begin{algorithm}
	\caption{Our ZK-Friendly Robust Aggregation}\label{alg:main_algorithm}
	\begin{algorithmic}[1]
		\REQUIRE Number of clients $\numClients$, threshold $\THD\in(0,1)$, client weights $\left(\model_i\right)_{i\in\numClients}$, commitments $\left(\com_{i,j}\,:j=1,\dots,\dimModel\right)_{i=1,\dots,\numClients}$ to normalized weight coefficients and commitments $\left(\comNorm_i\right)_{i=1,\dots,\numClients}$ to weight norms.
		\ENSURE Aggregated weights $\model$ and proofs $\left(\NIZK^{\ip}, \NIZK^{\THD}, \NIZK^{\textrm{agg}}\right)$.
		\color{Blue}
        \STATE \textbf{- Extension 1: input validation -}
	\STATE \textit{The clients provide a proof $\NIZK^{\mathsf{Norm2}}_i$ that their inputs and norms are consistent with the provided commitments.}
	\color{Black}
        \STATE \textbf{- Computing similarities -}
	\STATE $\aggModel \gets \sum_{i\in \numClients}\model_i$ \textcolor{gray}{\small\tt//scaled reference model}

		\FOR{each client $i\in\numClients$}
		\STATE $\cosSim_i \gets \cosSim(\aggModel,\model_i)$
		\ENDFOR
		\STATE $\mathcal{C} \gets (\cosSim_i)_{i\in\numClients}$
		\STATE Compute proof $\NIZK^{\ip}$ of correctness of the inner products

		\STATE \textbf{- Threshold selection and cluster division -}
		\STATE $\mathcal{C}_1 \gets \{\cosSim_i \in \mathcal{C} \mid \cosSim_i < \THD\}$ \textcolor{gray}{\small\tt//malicious inputs}
		\STATE $\mathcal{C}_2 \gets \{\cosSim_i \in \mathcal{C} \mid \cosSim_i \ge \THD\}$ \textcolor{gray}{\small\tt//benign inputs}
		\color{Blue}
		\STATE \textbf{- Extension 2: semi-static threshold -}
		\STATE Compute $\text{avg\_cos\_sim} \gets \frac{1}{\numClients}\sum_{i \in \numClients}\cosSim_i$
		\STATE $\mathcal{C}_1 \gets \left(\cosSim_i \in \mathcal{C} \mid \cosSim_i < \text{avg\_cos\_sim}\right)$
		\STATE $\mathcal{C}_2 \gets \left(\cosSim_i \in \mathcal{C} \mid \cosSim_i \ge \text{avg\_cos\_sim}\right)$
		\color{Black}
		\STATE Compute proof $\NIZK^{\THD}$ of correct threshold computation

		\STATE \textbf{- Aggregation phase -}
		\STATE $I \gets \{\,i \in \numClients \mid \cosSim_i \in \mathcal{C}_2\,\}$\label{alg-line:I}
		\STATE $\model \gets \frac{1}{|I|}\sum_{i\in I}\left(\model_i\right)$ \textcolor{gray}{\small\tt//final model}
		\color{Blue}
		\STATE \textbf{- Extension 3: weighted sum -}
		\STATE Compute $b_i \gets \frac{1}{1-\cosSim_i}$ for $i=1,\dots,\numClients$
		\STATE Compute $b \gets \sum_{i\in I}b_i$
		\STATE $\model \gets \sum_{i\in I}\left(\frac{b_i}{b} \cdot \model_i\right)$
		\color{Black}
		\STATE Compute proof $\NIZK^{\textrm{agg}}$ of correctness of final aggregation

		\RETURN $\model, \left(\NIZK^{\ip}, \NIZK^{\THD}, \NIZK^{\textrm{agg}}\right)$
	\end{algorithmic}
\end{algorithm}

We also refer to \Cref{fig:OurMethod} for an overview of the different phases explained in the following.
\begin{figure}
	\centering
	\includegraphics[width=\linewidth]{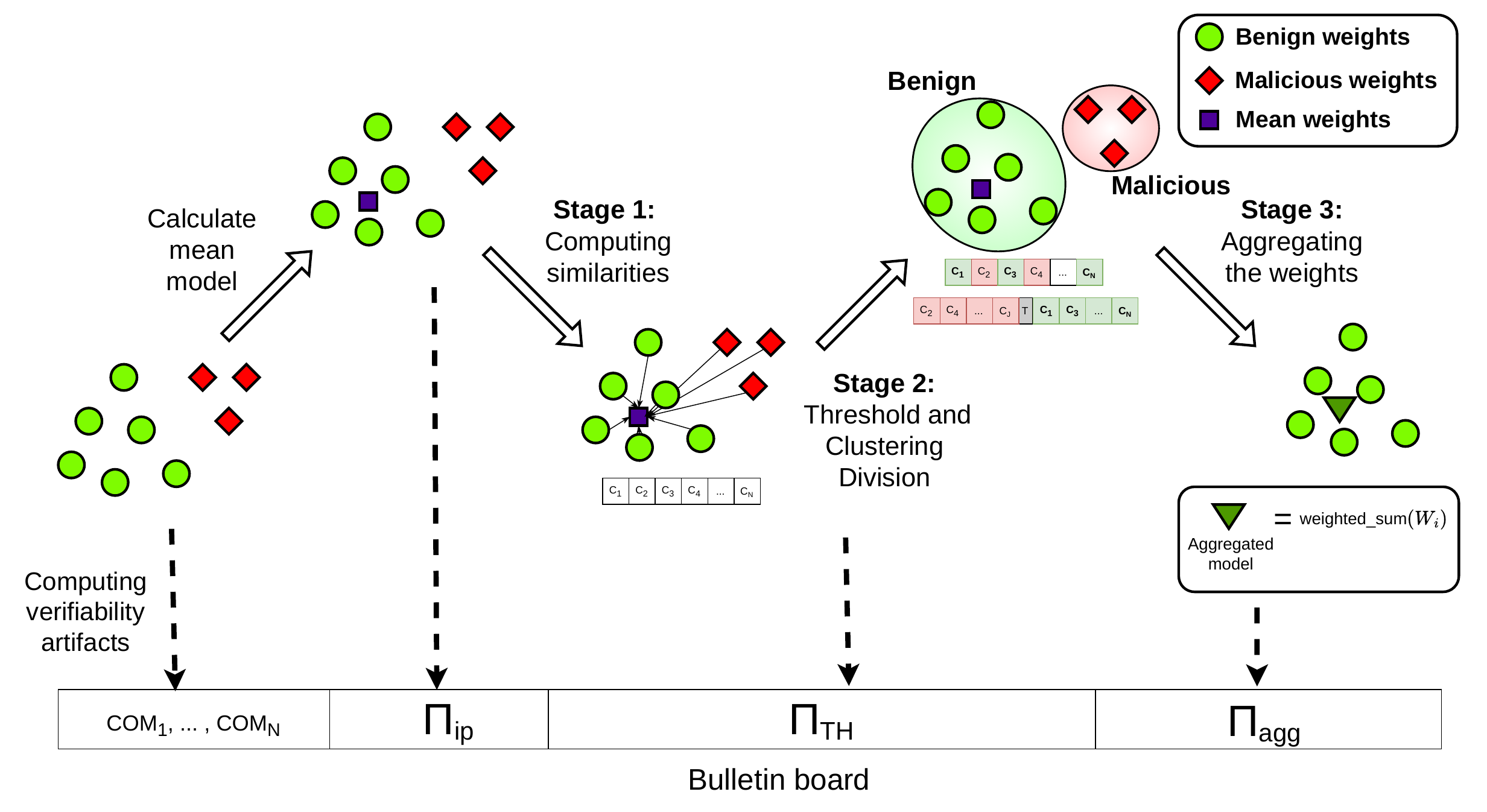}
	\caption{Pictorial description of our method.}
	\label{fig:OurMethod}
\end{figure}

\paragraph{Fixed-point arithmetic}
Before presenting our algorithm, we note that performing computations over floating-point numbers within zero-knowledge proofs is prohibitively expensive.
This is due to the complexity of accurately modeling floating-point semantics in arithmetic circuits, which leads to significant overhead.

To address this, we instead emulate fixed-point arithmetic.
Concretely, for a precision parameter $\precision \in \mathbb{N}$, we represent each real value $a \in \mathbb{R}$ by its scaled integer encoding $a' = \lfloor 2^{\precision} \cdot a \rceil$.
All computations are then carried out over these integer representations.
For ease of exposition, this transformation is left implicit and we are overloading notation whenever clear from the context.
For appropriately chosen parameters (cf.~\Cref{sec:discussion}), all computations can be carried out without overflow;
this can be ensured by standard bounds on the magnitude of intermediate values.

\paragraph{Modular design}
We adopt a modular approach in the presentation of our protocol.
This allows users to selectively enable or disable individual components when deploying our results, depending on the specifics of their application.
Moreover, this leads to clearer exposition and more comprehensible proof statements.
While combining all components into a single proof may yield minor efficiency gains (e.g., by avoiding redundant openings), we prioritize clarity and flexibility in our exposition.

\subsection{Input Preparation and Preprocessing}\label{sec:phase_input_preparation}

Input preparation consists of a series of steps as described next.

  \paragraph{Verifiability}
  To reduce the computational costs of verification, we let clients  submit their updated models in a certain format.
  That is, we require clients to submit a Pedersen commitment per normalized model update parameter as well as a commitment to the original norm of the model update, i.e., for a normalized model update of the form $\model_i=({\model{i,j}})_{j=1}^\dimModel$ and original norm $\ell_i$, they submit:
  \begin{align*}
  	\com_{i,j}&=\COM(\model_{i,j},\rand_{i,j})=\model_{i,j}\cdot G + \rand_{i,j}\cdot H\,.\\
    \comEll_{i}&=\COM(\ell_i,\rand_i^{(\ell)})=\ell_i\cdot G + \rand_i^{(\ell)}\cdot H\,.
  \end{align*}
  The aggregator additionally receives the openings $\rand_{i,j},\rand_i^{(\ell)}$ of the commitments, while the commitments $\com_{i,j},\comEll_i$ are published.

  Now, the aggregator checks that the norm of the normalized vectors indeed corresponds to $2^\precision$ (for some precision parameter $\precision$), and aborts otherwise, opening the malformed commitments to allow for verifying the malformed inputs.

\paragraph{Extension 1: Input Validation}
The description above makes some implicit non-collusion assumption between the aggregator and the clients, in the sense that the aggregator checks the norm of $\model_i$ and aborts if it does not correspond to $2^\precision$.

Overcoming this limitation can be achieved by the clients providing a proof that the individual coefficients are at most $2^\precision$ and the norm of the vector is close to this value (equality might not be achieved due to rounding errors in fixed-point arithmetic).

The proof goal is now given by:
\begin{align*}
	&\NIZK^{\mathrm{Norm}^2_i} \leftarrow
	\mathsf{NIZK}\bigg[\Big(
			(\model_{i,j}, \rand_{i,j},\nu_i,\rand_i, \drand_i^\nu)	:\\
	&\bigwedge_{j\in[\dimModel]} \com_{i,j} = \COM(\model_{i,j}, \rand_{i,j})\;\land\; \model_{i,j}\in[-2^\precision,2^\precision] &\land \\
	&\comNormSq_{i}=\sum_{j\in[\dimModel]}\model_{i,j}\cdot\com_{i,j} + H \cdot \drand_i^\nu\;\land\\
	&\comNormSq_{i}=\COM(\nu_i,\rand_i)\;\land\;\nu_i\in[L,R]
	\bigg]\,
\end{align*}
where the users also make the commitments $\comNormSq_i$ publicly available, and compute $\drand_i^{\nu}$ as $\drand_i^\nu=\rand_i-\sum \model_{i,j}\rand_{i,j}$. 

While not made explicit, all NIZKs follow best-practices in the sense that they are cryptographically bound to all necessary context information. 
That is, for our concrete instantiation, all Fiat--Shamir challenges are computed over the complete public statement, including all public inputs, system parameters, protocol and proof identifiers, session identifier, training-round identifier, and any other context information relevant to the respective proof.

It is easy to show that for a vector of reals with norm $1$, which is scaled by $2^\precision$ and rounded, the norm of the resulting vector lies in the following interval:
\[
	[L,R] = \left[\left(2^\precision-\frac{\sqrt{\dimModel}}{2}\right)^2,\, \left(2^\precision+\frac{\sqrt{\dimModel}}{2}\right)^2\right]
\]

This proof goal can, e.g., be instantiated using Bulletproofs.
Note that the $\dimModel+1$ range proofs can be batched~\cite{BunzBulletproofs2018}, resulting in an overall proof size of about $\mathcal{O}(\log_2(N)+\log_2(\precision))$ group elements.

We stress that, here and throughout the paper, all witnesses and intermediate values are restricted to publicly defined integer ranges. If the parameters, in particular $\precision$, are chosen as discussed in \Cref{sec:Results}, then every integer expression occurring in the proof relations has absolute value strictly smaller than $q/2$. 
Consequently, equality modulo the group order $q$ is equivalent to equality over the integers, and all range and comparison statements have their intended integer semantics.

\paragraph{Data sharing between clients and aggregator}
The aforementioned publication of the required inputs can be realized in different ways.
For instance, if clients are mutually known and jointly audit the aggregator, a (potentially anonymous) broadcast channel may be used.
For ease of exposition, however, we assume that data is shared via a \emph{bulletin board}, which also enables third-party auditing, cf.\ also \Cref{fig:OurMethod}.
More specifically, both clients and the aggregator have write access to the bulletin board, while any eligible verifier has read access.
The fundamental requirement for such a bulletin board is that it is \emph{append-only}, meaning that once data is published, it cannot be modified or deleted.
For concreteness, the reader may think of a blockchain as a canonical instantiation.

\paragraph{Input authentication}
While not being the central focus of this paper, a natural question is how to authenticate the inputs published by clients.
A straightforward solution is to require clients to sign their commitments before publication.
However, this approach relies on a public key infrastructure to identify eligible participants and reveals which client participates in each training round.

An alternative is to employ advanced signature schemes such as group signatures~\cite{BellareMW03,KrennSS25} without opening.
In this setting, a group manager issues signing keys to clients, allowing verifiers to confirm that a signature originates from a legitimate group member while preserving signer anonymity.
To prevent multiple submissions from a single (potentially malicious) client, one can use group signatures with controlled (or context-specific) linkability~\cite{GarmsL19,KrennSS19,ZhangLLYAW19}.
These schemes enable detection of multiple signatures from the same signer within a given training round, while maintaining anonymity across different rounds.

\paragraph{Avoiding replays and mix-and-match attacks}
Federated learning typically consists of multiple training rounds.
It is therefore important that the inputs and outputs of each round are uniquely determined to prevent, e.g., inputs from previous rounds being reused later.
To this end, we require that the aforementioned signatures authenticate not only the inputs but also all necessary context information (e.g., session identifier, round number, etc.) to prevent so-called mix-and-match attacks.

\medskip

However, while being central to the overall security of the overall construction, we will not further elaborate on these aspects, as there exists a variety of known solutions in the literature which can be modularly integrated and used.

\subsection{Computing Similarities}

Our algorithm next computes the cosine similarity between the provided inputs, in order to quantify their closeness and to obtain a basis for inclusion/exclusion decision.
Computing pairwise similarities across all clients is impractical due to its quadratic cost in the number of clients, as it would be a 2-dimensional matrix.
Instead, we first compute the mean of the clients' weights to obtain the aggregate weights $\aggModel$.
We then compute a one-dimensional similarity vector $C$, which contains the similarity between a client and this $\aggModel$, i.e., $C = \bigl(\cosSim(\aggModel, \model_i)\bigr)_{i\in\numClients}$, as defined in (\ref{eq:cosimilarity}).

\paragraph{Verifiability}
  Based on the commitments and openings from the previous step, the aggregator computes:
  \[
  	\aggModel_j = \sum_{i=1}^\numClients \model_{i,j}\,,\quad
  	\oAvg_j = \sum_{i=1}^\numClients \rand_{i,j}\,,\quad
  	\cAvg_j = \sum_{i=1}^\numClients \com_{i,j}\,.
  \]

  Furthermore, the aggregator computes commitments $\com_i^\ip$ to the inner products $\ip_i=\sum_{j\in \dimModel}\aggModel_j \model_{i,j}$ as well as $\com_i^{\ip^2}$ to $\ip_i^2$, each using fresh randomness $\rand_i^{\ip}, \rand_i^{\ip^2}$, respectively.
  Moreover, it computes a commitment $\cAvg^{\mathrm{Norm}^2}$ to the norm of the average model $\aggModel$.
  Finally, it also computes the values $\drand_i$, $\drand_i'$ and $\drand$ which express the difference between $\com_i^\ip$, $\com_i^{\ip^2}$ and $\cAvg^{\mathrm{Norm}^2}$, respectively, when expressed as fresh commitments and when they are computed from commitments to other elements.

  All commitments are made available to the verifier (except for $\cAvg_j$ which the verifier can directly re-compute themselves), and the following zero-knowledge proof of knowledge, showing the correctness of all inner product computations, is computed as follows:

\begin{align*}
&\NIZK^{\ip} \leftarrow
\mathsf{\ZKP}\bigg[ \\
&\Big(
(\model_{i,j}, \rand_{i,j}), (\aggModel_j, \oAvg_j),
(\ip_i, \rand_i^{\ip}), (\ip_i^2, \rand_i^{\ip^2}),
\drand, \drand_i, \drand_i^\prime
\Big) :\\
    &\bigwedge_{j\in [\dimModel]} \cAvg_j = \COM(\aggModel_j, \oAvg_j) \; \land \\
    &\bigwedge_{i\in [\numClients]} \left(\com_i^{\ip} = \COM(\ip_i,  \rand_i^\ip)\; \land\;
     \com_i^{\ip} = \sum_{j\in[\dimModel]} \aggModel_j \com_{i,j} + \drand_i\cdot H \right)\land  \\
    &\bigwedge_{i\in [\numClients]} \left(\com_i^{\ip^2} = \COM(\ip_i^2, \rand_i^{\ip^2}) \; \land \;
     \com_i^{\ip^2} = \com_i^{\ip} \cdot \ip_i + H \cdot \drand_i^\prime \; \right)\land \\
    &\cAvg^{\mathrm{Norm}^2} = \sum_{j\in[\dimModel]} \aggModel_j\cAvg_j + H \cdot \Delta \bigg]
\end{align*}

Here and in the following, to ensure completeness -- i.e., preventing a malicious aggregator from intentionally omitting valid inputs -- the verifier, as part of proof verification, also validates the proof goal itself, ensuring that all validly signed commitments for the specific context and training round are included in the zero-knowledge proof computation.
Consequently, the aggregator cannot omit, replace, or add client inputs without invalidating the proof.
Together with the context-specific linkability mechanism from input authentication, this also prevents a participant from contributing multiple inputs to the same training round.

In order to enable these checks, we assume the existence of an unambiguous round-finalization mechanism. 
Once a round is finalized, the set of inputs for that round consists of all validly authenticated submissions published before the corresponding deadline.

\subsection{Clustering and Inclusion Decision}
In order to now cluster the clients into benign and malicious, the aggregator computes a threshold $\THD$ that partitions the similarity vector $C$ into two sets. Clients are then classified according to this partition: inputs with cosine similarity at least $\THD$ are considered benign, whereas values below $\THD$ are considered malicious.

While in a very basic version, $\THD$ could be thought of as a fixed value, setting $\THD=\text{mean}(C)$ allows for basic adjustments (e.g., to account for increasing similarities over multiple training rounds).

\medskip

Here and in the following, we assume $\THD$ to be a public decision parameter, for example specified in a public aggregation policy.
For ease of exposition, we assume that $2^\precision \THD \in \mathbb{N}$, so that all subsequent bounds are expressed over the integers.

For dynamically computed thresholds, this condition might not be satisfied directly.
In this case, we round the threshold to the nearest value on the fixed-point grid, i.e., ${\THD'}=2^{-\precision}\left\lfloor 2^\precision\THD \right\rceil$, which introduces an absolute approximation error of at most $\left| \THD - {\THD'} \right| \leq 2^{-\precision-1}$.
All subsequent comparisons are then performed with respect to ${\THD'}$.
Consequently, the approximation can affect the classification only of inputs whose similarity lies within $2^{-\precision-1}$ of the original threshold, which can be considered negligible in practice.

\paragraph{Verifiability}
  While conceptually simple, the above logic of labeling clients with sufficiently large cosine similarity as benign is non-trivial in the cryptographic domain.
  This is because no entity---including the clients themselves---should be able to later check whether or not their inputs were rejected or not.

  Therefore, the algorithm first computes, for accepted inputs, $\com_i^\THD$ as a commitment to $1$, for rejected inputs it is computed as a commitment to $0$, without disclosing the contained bit $b_i$.
  These values will later serve as means for setting the weight of the corresponding input to $0$ by multiplication with $b_i$.

  More specifically, we set $\com_i^\THD=\COM(1,\rand_i^\ip)$ if $\mathrm{cos\_sim}(\aggModel,\model_i)\geq\THD$;
  otherwise, we set $\com_i^\THD=\COM(0,\rand_i^\ip)$, for uniformly random elements $\rand_i^\ip$.

  For algebraic constraints (e.g., computation of integer divisions or square roots), we rewrite the requirement
  \[
    \mathrm{cos\_sim}\left(\aggModel,\model_i\right)
= \frac{\left\langle \aggModel, \model \right\rangle}{\left\|\aggModel\right\|_2 \cdot \left\|\model_i\right\|_2} \geq \THD
  \]
  to
  \begin{equation}\label{eq:threshold}
    \left\langle\aggModel, \model_i\right\rangle^2 - 2^{2\precision}\cdot\THD^2\cdot\left\|\aggModel\right\|_2^2  \geq 0 \,\land \, \left\langle\aggModel, \model_i\right\rangle  \geq 0\,,
  \end{equation}
  thereby also using the fact that all clients' inputs are normalized vectors of norm $2^\precision$ (up to the precision given by the fixed-point arithmetic).\ifsubmission%
	  \footnote{More precisely, the norm of the vectors lies in the interval $\left[2^{\precision}-\sqrt{\dimModel}/2, 2^{\precision}+\sqrt{\dimModel}/2\right]$.
	  Rounding the norm to exactly $2^{\precision}$ thus effectively evaluates the client models based on their fixed-point representation with precision $\precision$ bits.}
\fi
  \,Note that the second clause is necessary due to the performed squaring in the rewriting of the inequality.
  Also note that the required norms and inner products exactly correspond to the values hidden inside $\com_i^{ip},\com_i^{ip^2}$, and $\cAvg^{\mathrm{Norm}^2}$, respectively.

  The following zero-knowledge proof statement now ensures that the bit committed in $\com_i^\THD$ was indeed computed in the correct manner, where the first clause covers the case in which (\ref{eq:threshold}) was satisfied, while the last two clauses represent the case where either of the two conditions was violated.

\begin{flalign*}
& \NIZK^{\THD} \leftarrow
\mathsf{NIZK}\Big[
(\ip_i, \rand_i^{\ip}),
(\ip_i^2, \rand_i^{\ip^2}),
(B_i, \rand_i^B),
\rand_i^{\THD}
:
\\
& \big(
\com_i^{\ip^2} = \COM(\ip_i^2, \rand_i^{\ip^2})
\wedge
\com_i^{\THD} = \COM(1, \rand_i^{\THD})
\\
& \wedge\;
\COM(B_i, \rand_i^B)
=
\com_i^{\ip^2}
- \cAvg^{\mathrm{Norm}^2}
\cdot \THD^2 \cdot 2^{2\precision}
\\
& \wedge\;
\com_i^{\ip} = \COM(\ip_i, \rand_i^{\ip})
\wedge
\ip_i \geq 0
\wedge
B_i \geq 0
\\[0.8em]
& \boldsymbol{\vee}
\\[0.3em]
& \com_i^{\ip^2} = \COM(\ip_i^2, \rand_i^{\ip^2})
\\
& \wedge\;
\COM(B_i, \rand_i^B)
=
\com_i^{\ip^2}
- \cAvg^{\mathrm{Norm}^2}
\cdot \THD^2 \cdot 2^{2\precision}
\\
& \wedge\;
C_i^{\THD} = \COM(0, \rand_i^{\THD})
\wedge
B_i < 0
\\[0.8em]
& \boldsymbol{\vee}
\\[0.3em]
& \com_i^{\THD} = \COM(0, \rand_i^{\THD})
\wedge
\com_i^{\ip} = \COM(\ip_i, \rand_i^{\ip})
\wedge
\ip_i < 0
\big)
\Big]
\end{flalign*}

\paragraph{Extension 2: Semi-static Thresholds}
  The description above assumed a fully static threshold $\THD$ for cosine similarity;
  we show how to achieve $\THD=\text{mean}(C)$.

  This can be achieved by leveraging the already computed commitments $\com_i^{\ip}$, as
  \[
    \com'=\sum_{i\in[\numClients]}\com_i^\ip
  \]
  is a publicly computable commitment to the sum $s$ of the inner products.
  This can now be turned into a commitment $\com''$ to the average cosine similarity $u$ by dividing it by $\numClients\cdot\|\aggModel\|_2$.
  Considering the fixed-point arithmetics we are emulating to avoid zero-knowledge proofs over the reals, the correctness of this division can now be proven by showing that $s\approx u\cdot\numClients\cdot\|\aggModel\|_2$ (up to rounding errors), i.e., by proving that
$$
    u\cdot\numClients\cdot\|\aggModel\|_2 - s\in [-B,B] \text{ for }
B = \frac{N^2}{2}\left(2^\precision+\frac{\sqrt{\dimModel}}{2}\right)
$$
  using standard multiplication proofs.
  The bounds in the computation of $\prod_i^{\THD}$ can then be rewritten in a canonical way.

  The main additional costs of this step are dominated by a single two-sided interval proof, which is negligible compared to the remaining costs of the protocol.

Nevertheless, to obtain a more accurate threshold, the appendix presents an inter-means classification algorithm that remains ZK-friendly but is more expensive due to its iterative nature.
The clustering decision is made in the same way as before.

Additionally, we consider the degenerate case in which all clients are benign and therefore all should be included in the aggregation;
otherwise, the procedure could unnecessarily exclude up to half of benign clients.
Since the baseline method always produces a two-cluster partition, in the appendix we also detail a ZK-friendly method for determining whether a two-way split is necessary or whether we should instead retain a single cluster.

\subsection{Final Aggregation of Weights}
Once the benign cluster has been identified, we apply FedAvg for aggregating the weights, i.e., the final model is simply computed as the average of the benign inputs.

Nonetheless, as suggested by previous works~\cite{nguyen2022flame,campos2025flaegis}, robustness can be further improved by using a robust aggregation function to mitigate imperfect classification.
Accordingly, we consider an aggregation function that weights each client's contribution based on the distance between its similarity score and $1$, reducing the influence of residual malicious clients and the impact of misclassification.
However, many robust aggregation functions (e.g., the median) are expensive to verify. Therefore, we seek an aggregation method that is both lightweight and robust. To this end, we propose a similarity-weighted aggregation function based on the distance to $1$ in the similarity vector $C$. Specifically, each client $i$ is assigned a weight $b_i = \frac{1}{1 - c_i}$, and the aggregated model is computed as
\begin{equation}
\model = \sum_{i \in C_2} \frac{b_i}{b}\model_i,
\end{equation}
where $b = \sum_{i \in C_2} b_i$.
This weighting scheme down-weights clients with low similarity (smaller $c_i$), thereby limiting their influence, while assigning higher impact to clients whose updates are closer to $\aggModel$.

\paragraph{Verifiability}
  The algorithm now obliviously aggregates models close to the mean model, while distant models are not considered.
  That is, the $j^{\text{th}}$ parameter of the final model is computed as $\sum_{i\in[\numClients]}\gamma_i\cdot\model_{i,j}$, where $\gamma_i$ is equal to the product $\delta_i\ell_i$.
  Recall that as detailed in \cref{sec:phase_input_preparation}, $\ell_i$ is the norm of the $i^{\text{th}}$ client's model.
  The commitment $\com_i^{\alg{final}} $ to this value is computed accordingly.
  The correctness of this aggregation process is now shown via the following statement:

\begin{align*}
&\NIZK^{\agg'} \leftarrow
\mathsf{NIZK}\bigg[\Big(
(\delta_i, \rand_i^{\THD}), (\ell_i, \rand_i^{\ell}),
    (\gamma_i, \rand^\gamma_i),  \drand_i^{\prime \prime}, \drand_j^{\prime \prime  \prime} \Big)
 :\\
    &\bigwedge_{i\in[\numClients]} \com_i^{\THD} = \COM(\delta_i, \rand_i^\THD) \; &\land \\
    &\bigwedge_{i \in[\numClients]} \comNorm_{i} = \COM(\ell_i, \rand_i^{\ell}) \; &\land  \\
    &\bigwedge_{i \in[\numClients]} \com_i^{\alg{tmp}} = \delta_i\cdot\comNorm_{i} + H \cdot \drand^{\prime\prime}_i \; &\land \\
    &\bigwedge_{i \in[\numClients]} \com_i^{\alg{tmp}} = \COM(\gamma_i, \rand_i^\gamma) \; &\land\\
    &\bigwedge_{j\in[\dimModel]} \com_j^{\alg{final}} = \sum_{i \in [\numClients]} \gamma_i\cdot\com_{i,j} + H \cdot \drand^{\prime\prime\prime}_j
\bigg]
\end{align*}

The commitments $\com^{\alg{final}}_j$ obtained above contain the scaled sum of all selected model updates.
If the number of selected inputs, i.e., $|I|$ (cf. also \Cref{alg:main_algorithm}, Line \ref{alg-line:I}), were public, the aggregator could simply open these commitments and the clients could divide the resulting coordinates by $|I|$.
However, disclosing $|I|$ would also reveal the number of rejected inputs, which we aim to keep private.

\medskip

We therefore additionally commit to the cardinality of the selected set.
By the additive homomorphism of the commitment scheme, the verifier can compute
\[
  \comCard := \sum_{i\in[N]} \com^{\THD}_i = \COM\left(|I|,\sum_{i\in[N]}\rand^{\THD}_i\right).
\]
Thus, $\comCard$ is a publicly computable commitment to the number of selected inputs without revealing this number.

For each coordinate $j\in[d]$, let the sum of the correctly scaled selected models, i.e., the value committed to inside $\com_j^{\alg{final}}$, be
\[
  S_j := \sum_{i\in[N]} \gamma_i \model_{i,j}\,.
\]
The aggregator computes the rounded average
$
  z_j := \left\lfloor \frac{S_j}{|I|} \right\rceil
$
and publishes a commitment
\[
  \com^{\alg{avg}}_j := \COM(z_j,\rand^{\alg{avg}}_j).
\]
Correctness of the division requires to show that  $\left|z_j-{S_j}/{|I|}\right|\leq 1/2$, which is equivalent to:
\[
  1\leq |I|\leq N
  \qquad\text{and}\qquad
  -|I|\leq 2(|I|z_j-S_j)\leq |I|,
\]
Importantly, all these relations can be proven without revealing $|I|$.

More precisely, setting $m=|I|$, this can be shown by extending the above proof goal by the following statement:
\begin{align*}
  &\NIZK^{\alg{avg}}\gets\mathsf{NIZK}\Big[
  \left(m,\rand^{\alg{card}}),
    (S_j,\rand^{\alg{final}}_j),
    (z_j,\rand^{\alg{avg}}_j),
    (r_j,\rand^r_j)_{j\in[d]}\right):\\
  &\comCard=\COM(m,\rand^{\alg{card}})
  \ \land\ 
  m\in[1,N]\ \land\\
  &\bigwedge_{j\in[\dimModel]}
  \Big(
    \com^{\alg{final}}_j=\COM(S_j,\rand^{\alg{final}}_j)
    \ \land\
    \com^{\alg{avg}}_j=\COM(z_j,\rand^{\alg{avg}}_j)\ \land
\\
  &\com^r_j=\COM(r_j,\rand^r_j)
    \ \land\
    r_j=mz_j-S_j\ \land
\\
    & m+2r_j\in[0,2N]
    \ \land\
    m-2r_j\in[0,2N]
  \Big)
\Big].
\end{align*}

The aggregator then opens the commitments $\com^{\alg{avg}}_j$ to the clients.
The resulting vector is still scaled by $2^\precision$, which can be removed
locally using the public precision parameter.
Neither the selected set $I$ nor its cardinality $m$ is revealed.

The algorithm finally outputs $\NIZK^{\alg{agg}}=(\NIZK^{\alg{agg'}},\NIZK^{\alg{avg}})$.

\paragraph{Extension 3: Weighted aggregation}
  Again taking into consideration the fixed-point arithmetic we are emulating, we first prove that $b_i$ is the correct computation result, i.e., that $b_i=\lfloor \frac{1}{1-c_i}\rceil$, by showing that $|b_i(1-c_i)-1|$ is small, or, equivalently, that 
  $$
    b_i(1-c_i)-1 \in (-2^{-\precision}(1-c_i),2^{-\precision}(1-c_i)).
  $$
  
  After considering the scaling and the way that $c_i$ is computed, this results in committing to
  \[
  \hat{b}_i=2^\precision\frac{2^\precision}{2^\precision-\hat{c}_i}\quad\text{for}\quad\hat{c}_i=\frac{\langle\aggModel,\model_i\rangle}{\|\aggModel\|_2},
  \]
  and proving the corresponding precision, i.e., that
  \[
    \hat{b}_i(2^\precision-\hat{c}_i)\in\left(2^{2\precision}-(2^\precision-\hat{c}_i),2^{2\precision}+(2^\precision-\hat{c}_i)\right)\,.
  \]
  Subsequently, the divisor $\sum b_i$ can be computed as the sum of the commitments, and the final model can be computed by enriching the final aggregation statement by another multiplicative factor $\hat{b}_i$.

  \medskip
  
  Note that for simplicity, we did not explicitly model the boundary case $c_i=1$, in which the denominator vanishes and the computation aborts.
  This case requires the normalized client update to be exactly aligned with the reference model and is therefore negligible in practice for reasonable precisions and model sizes.
  Yet, if needed, this could trivially be considered by introducing a small public regularization offset $\varepsilon$ (e.g., $\varepsilon=2^{-16}$) and computing $b_i=\frac{1}{1-(c_i-\varepsilon)}$.

  \medskip

  We omit making the proof goal explicit as it follows the same structure as those before without providing additional insights.

  From a complexity point of view, we note that the costs are dominated by one two-sided range proof per client, which is relatively minor compared to the main proof goal when proving cosine similarity of the models.

\subsection{Security Analysis}
We now argue that the proposed verification layer provides the two security guarantees stated in Section~\ref{s:security_model}, namely aggregation soundness and public-transcript privacy, with respect to the scheme summarized in \Cref{alg:main_algorithm}.

For a fixed training round, let $\mathcal{S}$ denote the finalized set of validly authenticated client submissions.
The set $\mathcal{S}$ is determined by the bulletin-board and round-finalization mechanisms requested in~\Cref{sec:phase_input_preparation}.
The verifier constructs all proof statements from $\mathcal{S}$ and the public aggregation policy, including the session identifier, training round, threshold parameters, fixed-point precision, and all other context information.
Consequently, proofs generated for different rounds, sessions, or input sets cannot be combined.

The following theorem summarizes the security of the protocol.

\begin{theorem}[Aggregation soundness and public-transcript privacy]
\label{thm:security}
Assume that:
\begin{enumerate}
    \item the employed commitment scheme is computationally binding and
    hiding;
    \item the employed proof systems are complete, zero-knowledge, and
    knowledge-sound;
    \item the Fiat--Shamir transformation is secure in the random-oracle
    model;
    \item the input-authentication mechanism is unforgeable, context-bound, and prevents more than one valid submission by the same eligible client within a training round; and
    \item the bulletin board is append-only and provides an unambiguous finalized input set for every training round.
\end{enumerate}
Furthermore, assume either that the aggregator does not collude with a client in accepting an incorrectly normalized input, or that the input validation proofs $\NIZK^{\mathsf{Norm^2}}_i$ from
~\Cref{sec:phase_input_preparation} are enabled.

Then the above protocol satisfies aggregation soundness and public-transcript privacy as defined in \Cref{s:security_model}, except with negligible probability.
\end{theorem}

For the proof, we refer the reader to \Cref{app:proof_sec}.

\section{Experimental Analysis}\label{sec:Results}

We evaluate our pipeline under representative poisoning attacks and compare it against several robust FL defenses. We report two robustness metrics: (i) client-detection accuracy, defined as the fraction of clients correctly classified as benign or malicious for defenses that output client-level decisions, and (ii) final global-model accuracy, measured on benign clients. We also quantify public-verification overhead in terms of proof-generation time, verification time, and proof size.

\subsection{Settings}

The datasets and complete configuration details of our robust and verifiable FL scenario are provided in Appendix \ref{app:settings}. We implement the FL experiments using Flower \cite{flower}. The datasets used in this work are FEMNIST~\cite{caldas2018leaf}, Synthetic~\cite{caldas2018leaf}, and Sentiment140~\cite{go2009twitter}, all of which are classification datasets. FEMNIST is an image dataset for handwritten characters, Synthetic is an artificially generated tabular dataset with five classes, and Sentiment140 classifies tweet sentiment. For FEMNIST, we train a CNN for 30 rounds, one epoch per round, with 30 clients. For Synthetic, we train an MLP for 30 rounds, one epoch per round, with 20 clients. Finally, for sentiment140, we train an LSTM for 10 rounds, using one epoch per round, with 20 clients.

The attacks considered are label flipping, LIE, STATOPT, Mimic, Min-max, and Min-sum. Their description and selected parameters are provided in Appendix \ref{app:poison}. We consider four malicious-client detection methods (FLAegis \cite{campos2025flaegis}, Lomar \cite{li2021lomar}, SafeFL \cite{dou2025toward}, and M3D-FL \cite{atia2025m3d}) and two robust aggregation functions (median and Krum), covering a broad range of representative defenses for comparison. 
These baselines are not designed for public verifiability; nevertheless, we include them to assess the robustness-verifiability trade-off. In particular, they provide non-verifiable robustness references against which we compare the accuracy cost of restricting our pipeline to proof-friendly operations. Furthermore, Krum requires the expected number of Byzantine clients as an input, so we provide the true malicious-client count in each configuration.

Finally, in the experiments, we evaluate the core protocol only, namely fixed-threshold filtering and standard aggregation over the selected clients. The optional extensions are not used in the accuracy experiments; we report their estimated cryptographic cost separately.

\subsection{Results}
We now evaluate our ZK-friendly robust aggregation against the considered poisoning attacks and compare it with the baseline defenses. To assess the performance of each method, we report two metrics: \emph{detection accuracy}, i.e., the ability to label benign clients as benign and malicious clients as malicious, and \emph{final accuracy}, measured on benign clients only, which captures the accuracy achieved by the clients in their classification task. For each dataset we take attack-free FedAvg as the \textit{ideal accuracy} reference ($0.8312$ on FEMNIST, $0.92$ on Synthetic, and $0.80$ on Sentiment140). A robust method should remain as close as possible to this reference under attack. In the main body of the paper, we show only the results of the FEMNIST dataset, while the Synthetic and Sentiment140 results are reported in Appendix \ref{app:restresults}. Additionally, after the FEMNIST results, we provide a table that summarizes the results across all three datasets. 

\subsubsection{Performance of the detection method}

\paragraph{Detection accuracy}
As shown in Figure~\ref{fig:femnist_comparison_det}, our detector remains highly effective across attacks and malicious-client ratios, and is the only method that does not collapse in any evaluated configuration. Detection is $1.0$ under Label~Flipping, LIE, STATOPT, Mimic, and Min-max, while Min-sum is the only attack showing degradation, from $\approx0.91$ at $10\%$ to $\approx0.81$ at $40\%$. The baselines are more attack-dependent. Lomar performs nearly perfectly under Mimic, Min-max, and Min-sum, but degrades under LIE ($\approx0.50$--$0.72$) and STATOPT. M3D-FL is relatively consistent, although detection generally remains between $0.6$ and $0.9$ and decreases as the malicious fraction grows. SafeFL performs well only in specific regimes, collapsing under LIE and Mimic and reaching only $\approx0.4$--$0.55$ under Min-max/Min-sum. Thus, baselines that outperform our detector in isolated configurations exhibit larger degradation in other settings. Our previous non-ZK-friendly method, FLAegis \cite{campos2025flaegis}, is the strongest competitor, achieving near-perfect detection under Label~Flipping and LIE and remaining strong under STATOPT. However, it degrades under Mimic and reaches only $\approx0.72$ at $10\%$ for Min-sum and Min-max. These detection errors are partly mitigated by its aggregation stage, resulting in stable final-model accuracy.

\paragraph{Final accuracy}
Figure~\ref{fig:femnist_comparison_acc} shows that final accuracy closely follows detection performance. Our method remains near the ideal FedAvg accuracy ($0.8312$) across all attacks and malicious ratios, typically around $0.81$--$0.83$, with low variance and no catastrophic failures. In contrast, M3D-FL collapses under LIE and STATOPT and strongly degrades under Mimic and Min-max. SafeFL performs well under Label~Flipping and STATOPT but drops to approximately $0.07$--$0.15$ under LIE, Mimic, and Min-max. Lomar is competitive under Mimic, Min-max, and Min-sum, but degrades under STATOPT and collapses under LIE. FLAegis remains the most stable baseline~\cite{campos2025flaegis}. Median and Krum also show stable results, generally around $0.78$--$0.81$, although they rarely outperform our method. Moreover, Krum requires \emph{a priori} knowledge of the number of Byzantine clients.

\begin{figure*}[!t]
\centering
\includegraphics[width=0.92\linewidth]{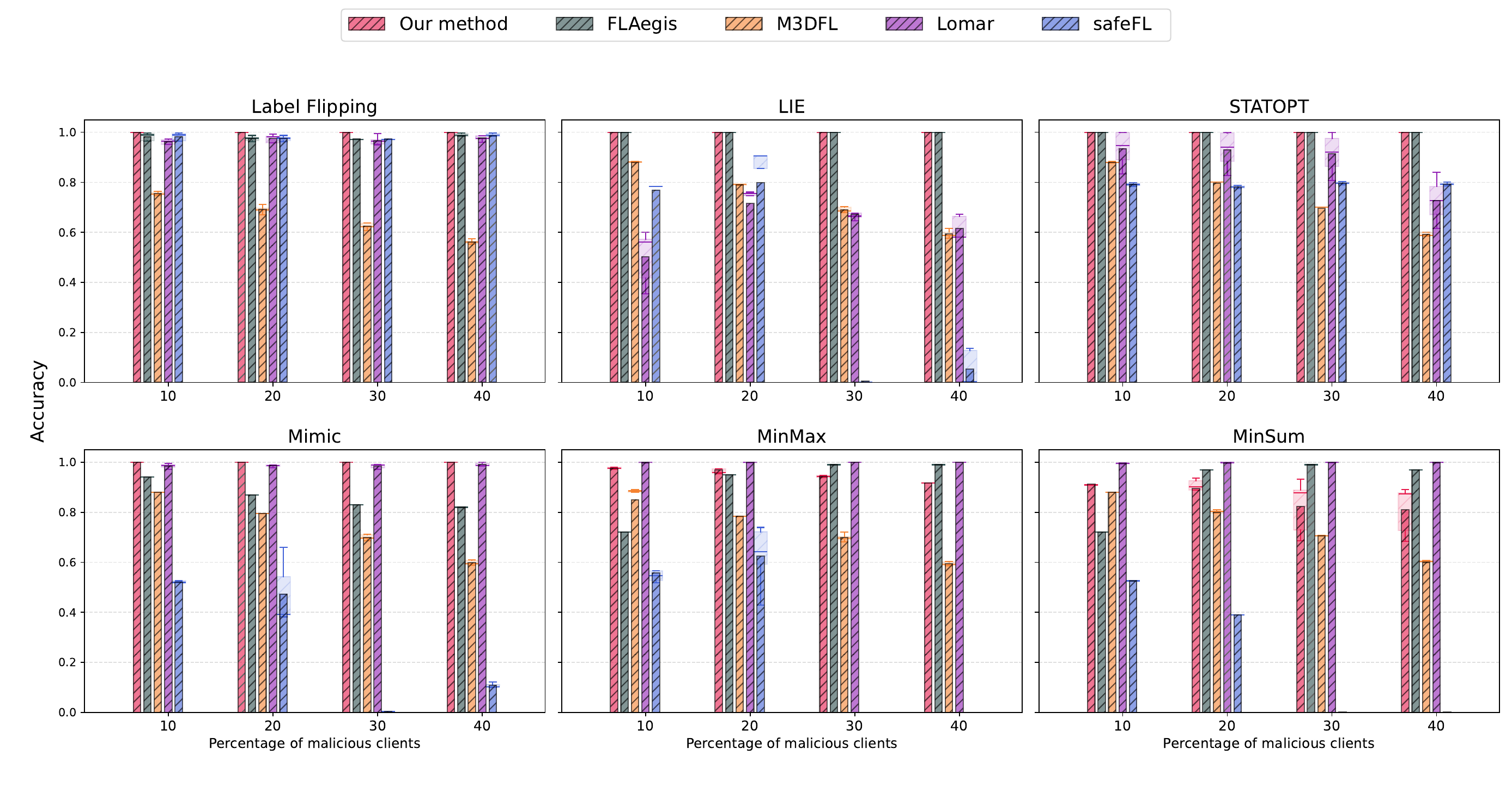}
\vspace*{-2.5em}
\caption{Detection accuracy of the different methods against the different poisoning attacks with different percentages for FEMNIST dataset.}
\label{fig:femnist_comparison_det}
\end{figure*}

\begin{figure*}[!t]
\centering
\includegraphics[width=0.92\linewidth]{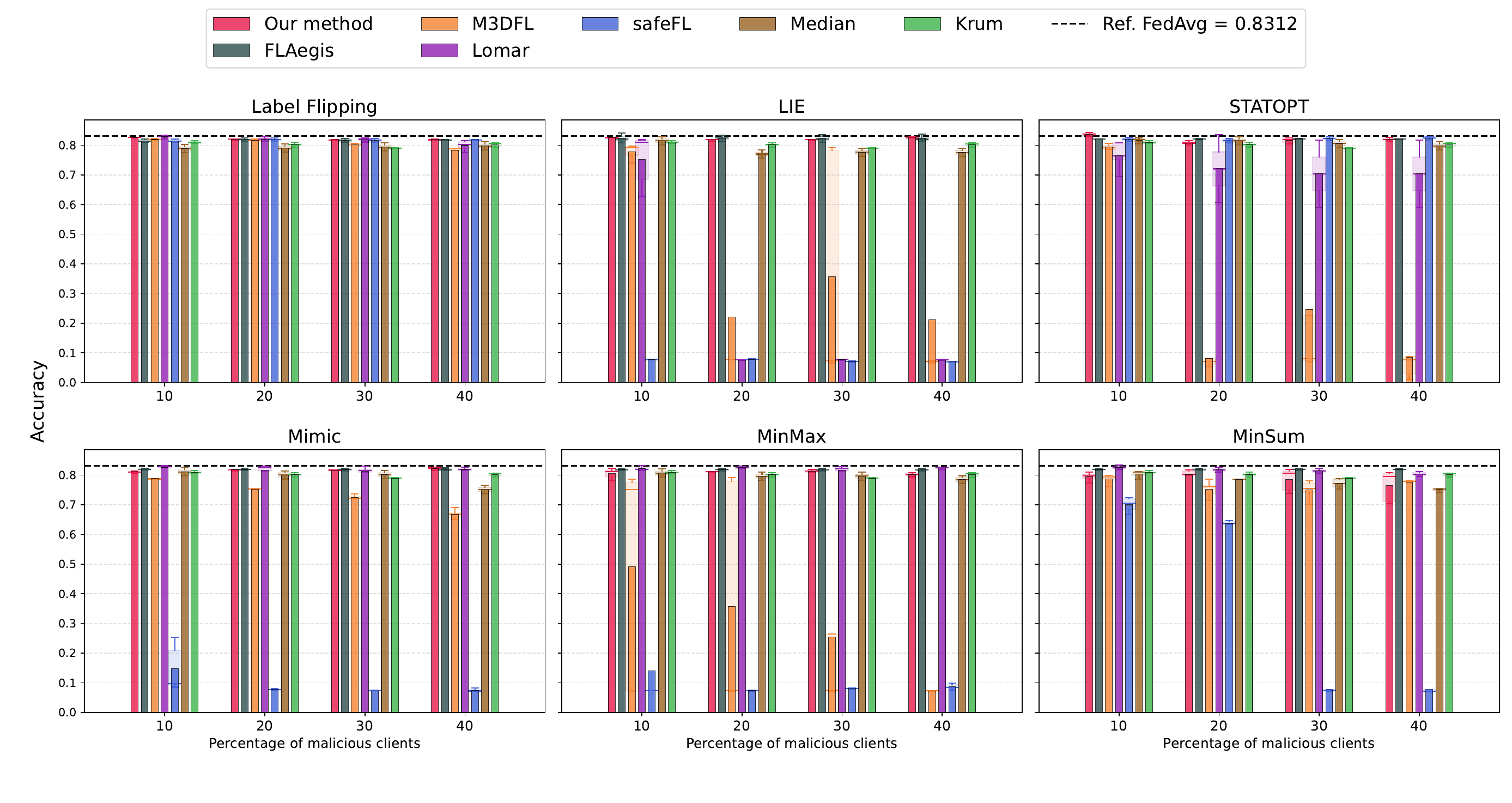}
\vspace*{-2.5em}
\caption{Final accuracy of the different methods against the different poisoning attacks with different percentages for FEMNIST dataset.}
\label{fig:femnist_comparison_acc}
\end{figure*}

\paragraph{Summary}
The remaining datasets, reported in Appendix \ref{app:restresults}, exhibit the same qualitative trends as \Cref{fig:femnist_comparison_acc} and \Cref{fig:femnist_comparison_det}. Although our method does not dominate every individual configuration, two aspects are important when interpreting these results.

First, our method is explicitly designed for verifiability. Its detection and aggregation rules are restricted to ZK-friendly linear and comparison operations so that every decision can be efficiently proven. In contrast, competing defenses may use more expressive operations, such as density-based clustering, spectral decompositions, or sorting-based statistics, which are considerably harder to verify in zero knowledge. Hence, the comparison deliberately places our method under a stricter computational constraint.

Second, excluding FLAegis, which represents our previous non-ZK-friendly work, our method provides the most consistent robustness across attacks. Other baselines may outperform it in particular configurations, but each suffers substantial degradation under at least one attack, whereas our method avoids catastrophic failures.

We quantify this behavior using \textit{Accuracy Degradation} (AD):
$\mathrm{AD} = Acc_{\text{ideal}} - Acc_{\text{obtained}}$, where $Acc_{\text{ideal}}$ denotes the no-attack FedAvg accuracy and $Acc_{\text{obtained}}$ the accuracy achieved under attack. Higher AD therefore indicates larger accuracy loss under attack. \Cref{tab:comparison} reports the AD averaged over attacks, malicious-client ratios and the three datasets.

FLAegis obtains the lowest average AD~\cite{campos2025flaegis}, but relies on operations that are impractical to verify in zero knowledge. Excluding FLAegis, our ZK-friendly method achieves the lowest AD among the evaluated defenses, outperforming the non-verifiable baselines. Median and Krum remain competitive but weaker, while detector-based approaches suffer substantially higher AD due to failures on specific attacks. In summary, these results show that strong robustness can be retained while restricting the aggregation pipeline to operations suitable for efficient zero-knowledge verification.

\begin{table}[t]
\centering
\small
\setlength{\tabcolsep}{4pt}
\renewcommand{\arraystretch}{0.85}
\begin{tabular}{lcc}
\hline
\textbf{Method} & \textbf{Mean AD} & \textbf{Std. dev.} \\ \hline
Our method & 0.0267 & 0.0086 \\
FLAegis    & 0.0182 & 0.0042 \\
M3D-FL     & 0.1952 & 0.1228 \\
Lomar      & 0.1041 & 0.1871 \\
SafeFL     & 0.2187 & 0.2186 \\
Median     & 0.0282 & 0.0060 \\
Krum       & 0.0414 & 0.0183 \\ \hline
\end{tabular}
\caption{Average accuracy degradation (AD) across attacks, malicious-client ratios, and the three datasets.}
\label{tab:comparison}
\vspace{-1.0em}
\end{table}

\subsubsection{Proof of concept}

We implemented a proof of concept of the scheme explained in \Cref{sec:robust_verif_fl} in Python relying on two cryptographic libraries:
\texttt{petlib}, which provides efficient primitives for elliptic-curve operations~\cite{petlib}, and \texttt{zksk}, which enables the construction of zero-knowledge proofs~\cite{lueks2019zksk}.
The implementation is based on elliptic-curve cryptography and uses the NIST/SECG 224-bit prime-field curve (\texttt{secp224r1}), offering an estimated 112-bit security level. All benchmarks were executed over 20 runs.
All the experiments were run in a virtual machine equipped with an Intel(R) Core(TM) Ultra 7 165H CPU (1.40 GHz) and 32~GB of RAM.


\begin{figure}[!t]
    \centering
    \begin{subfigure}{0.49\textwidth}
        \includegraphics[width=\linewidth]{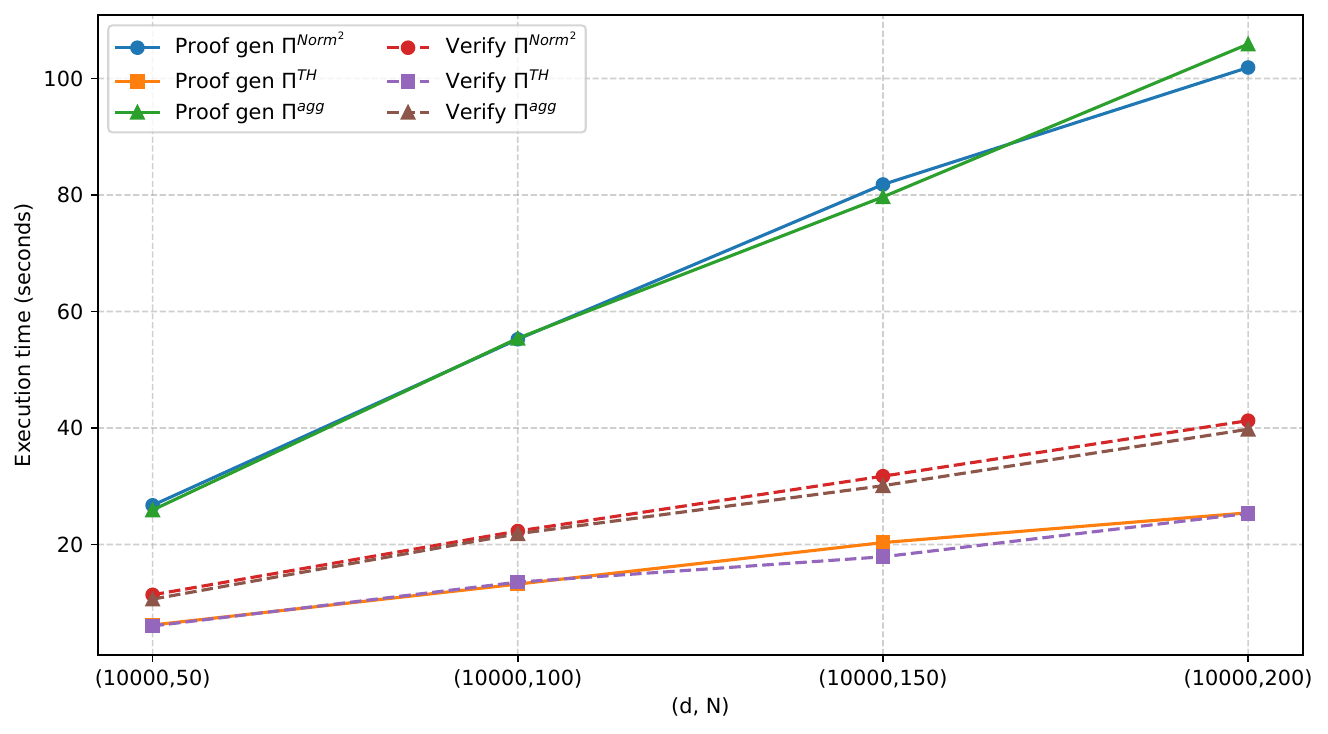}
    \end{subfigure}
    \hfill
    \begin{subfigure}{0.49\textwidth}
        \includegraphics[width=\linewidth]{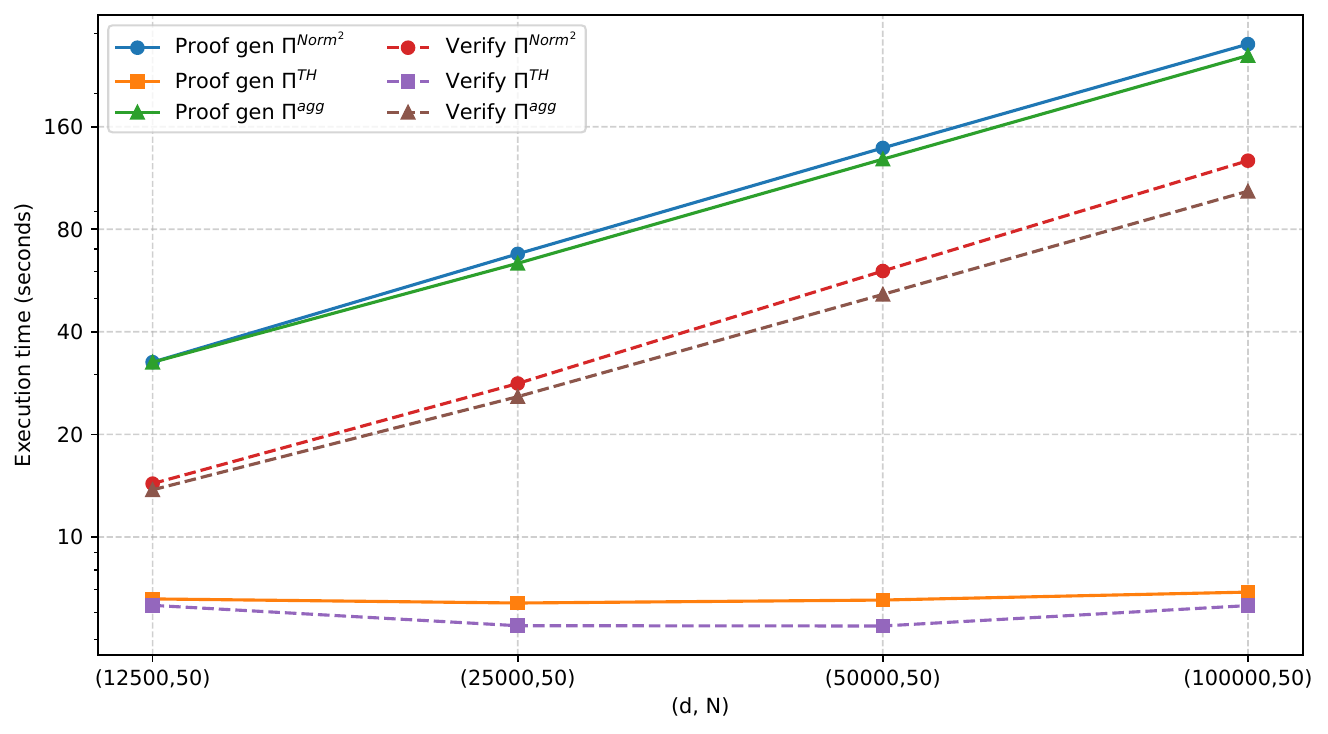}
    \end{subfigure}
\vspace*{-2.0em}
    \caption{Proof generation and verification times for 4 sets of parameters with a fixed model dimension $\dimModel$ (top) and a fixed number of clients $\numClients$ (bottom).}
    \label{fig:benchmark_time}
\end{figure}

\begin{figure}[!t]
    \centering
    \begin{subfigure}{0.49\textwidth}
        \includegraphics[width=\linewidth]{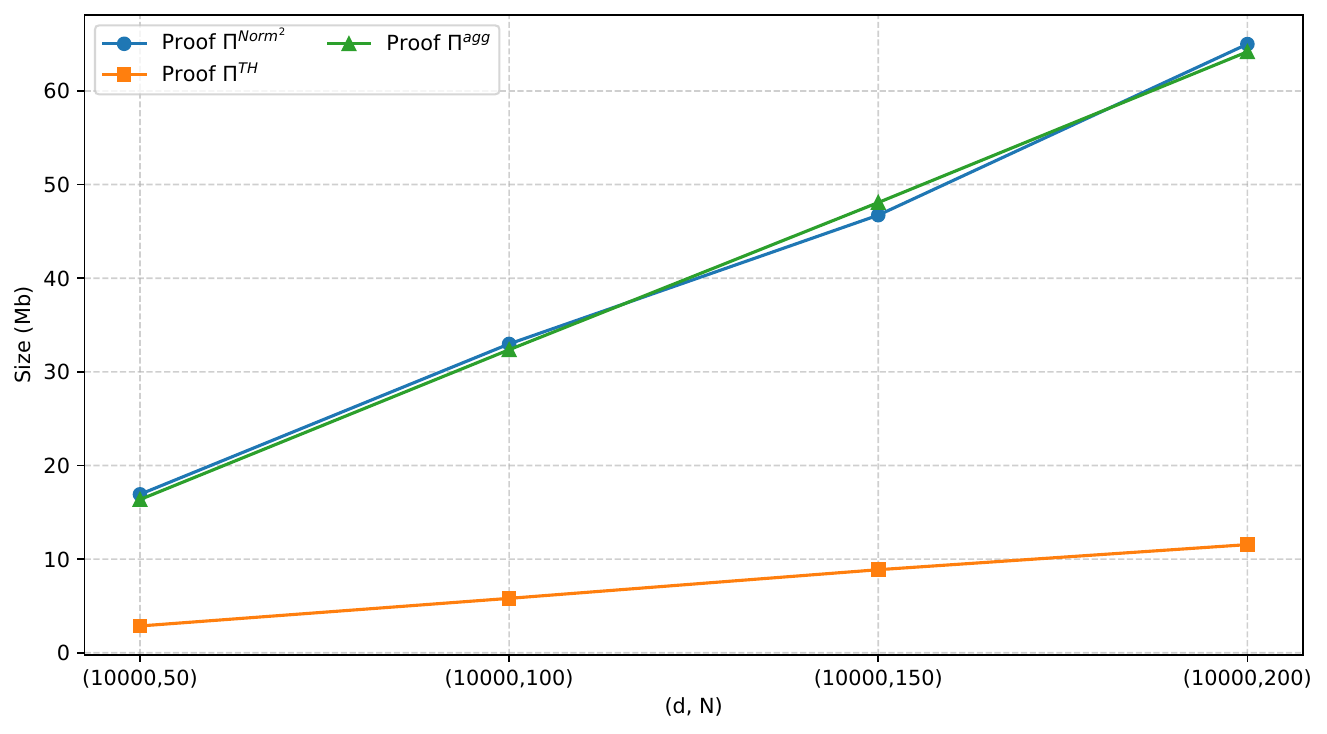}
    \end{subfigure}
    \hfill
    \begin{subfigure}{0.49\textwidth}
        \includegraphics[width=\linewidth]{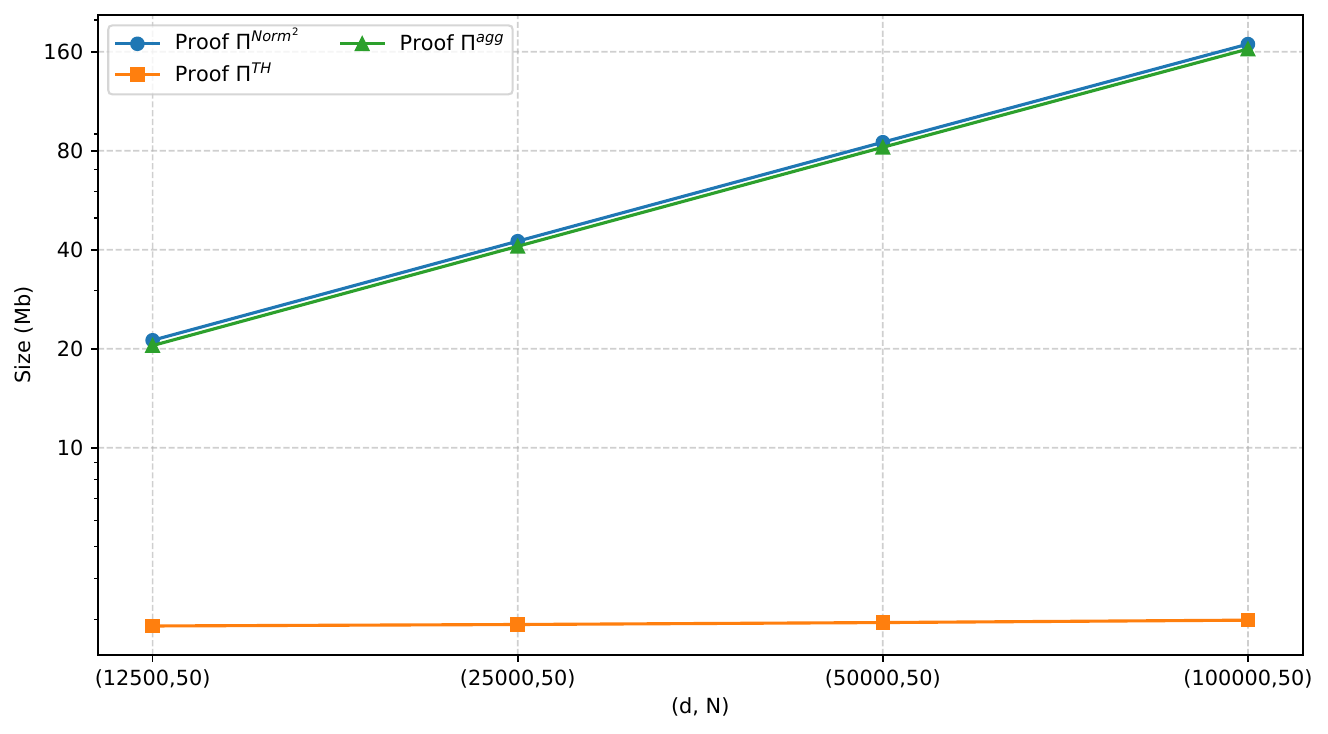}
    \end{subfigure}
\vspace*{-2.0em}
    \caption{Proof size for 4 sets of parameters with a fixed model dimension $\dimModel$ (top) and a fixed number of clients $\numClients$ (bottom).}
    \label{fig:benchmark_size}
\end{figure}

\Cref{fig:benchmark_time} and \Cref{fig:benchmark_size} report the performance of our verifiable aggregator implementation for four sets of parameters $(\dimModel,\numClients)$ where $\dimModel$ is the size of the models and $\numClients$ is the number of clients.
They respectively exhibit the proof generation/verification time and the proof size of the different NIZK components.
In both figures, we have fixed $\dimModel$ on the left and  $\numClients$ on the right.
We also have set the threshold to $0.903$ and the precision parameter $\lambda$ to $23$.
The theoretical size of the proofs $\NIZK^{\ip}$ and $\NIZK^{\agg}$ is expected to be asymptotically equal to $O(\dimModel\numClients)$ elements.
In $\NIZK^{\THD}$, the proof size should be asymptotically equal to $O(\numClients\lambda)$.

Our implementation shows that proof generation and verification times have different scaling behaviors depending on the proof type.
Both $\NIZK^{\ip}$ and $\NIZK^{\agg}$ scale bilinearly with $\dimModel$ and $\numClients$.
The same trend is observable for the size of the respective proofs.
In contrast, $\NIZK^{\THD}$ is largely independent of $\dimModel$ but grows linearly with $\numClients$, which is consistent with the fact that threshold proofs are generated independently for each participant

\subsubsection{Estimation of extensions efficiency}\label{sec:estimation}
In \cref{s:cost_estimates}, we provide a modular computational cost estimation for the main protocol and the introduced extensions. For typical FL parameter regimes, the overall cost is largely determined by the main protocol, while the extensions introduce additional modular overhead for the proof generated by the aggregator.
Furthermore, we there also give an estimation of the proving and verification times for different parameter sets.
While the costs of Extensions 2 and 3 are practically neglectable compared to the main protocol, validating the inputs may impose a significant verification overhead in case of a high number of clients.


\subsubsection{Parameter Estimation}\label{sec:discussion}

When instantiating the proposed protocol, the fixed-point precision must be chosen together with the group order.
All computations are carried out modulo the prime group order $q$.
To ensure that the resulting modular relations correspond uniquely to
the intended relations over the integers, every intermediate integer
value occurring in the proof statements must have absolute value
strictly smaller than $q/2$.

The largest intermediate value arises in the clustering step, namely
from the scalar expression represented by
$
  \bar{\com}^{\mathrm{Norm}^2}\cdot \THD^2\cdot 2^{2\precision}.
$
Its absolute value is upper-bounded by
$
  \dimModel^2\cdot \numClients^2\cdot 2^{8\precision}.
$
Consequently, the parameters must satisfy
$
  \dimModel^2\cdot \numClients^2\cdot 2^{8\precision} < q/2.
$

For example, consider $\numClients=2^{10}\approx 1\,000$, $\dimModel=2^{20}\approx 1\,000\,000$, and a group whose prime order $q$ has approximately $256$ bits.
The above requirement then yields
$$
  2^{2\cdot 20}\cdot
  2^{2\cdot 10}\cdot
  2^{8\precision}
  < q/2,
$$
and therefore permits $\precision= 24$.

Note that this bound only limits the fractional precision of the fixed-point encoding, not the precision of the corresponding integer operations.
Once the inputs are encoded, additions and multiplications are performed exactly on their integer representations since, by parameter selection, no wraparounds can occur.
Thus, these operations introduce no additional approximation error.
Additional rounding errors arise only at explicitly identified rescaling, division, normalization, or decoding steps.

\subsubsection{Potential Tradeoffs}

Our approach admits several trade-offs between privacy, efficiency, and robustness.
A first option is to trade privacy for efficiency. Our baseline design reveals no information about individual weight vectors beyond what can be inferred from the final aggregate. However, publishing intermediate values, such as the average weight vector, could enable more efficient techniques, including Pedersen vector commitments and inner-product arguments~\cite{inner_product_arguments}. This reduces bulletin-board and proof overhead, but increases information leakage and may facilitate inference attacks. Hence, the privacy impact of exposing intermediate aggregates must be carefully evaluated.

Similarly, revealing which clients were excluded as outliers simplifies verification by avoiding oblivious selection proofs. However, this exposes participation-level metadata, such as which clients were rejected and when, which may itself be sensitive.

The choice of commitment and zero-knowledge primitives also introduces efficiency trade-offs. Different schemes can realize the same functionality while favoring smaller proofs~\cite{BunzBulletproofs2018} or lower proving and verification time~\cite{plonk}.

Finally, the robustness layer involves a trade-off between statistical sophistication and verifiability. Methods such as coordinate-wise median, Krum, or geometric-median aggregation rely on sorting, pairwise distances, or iterative optimization, resulting in substantially larger verification statements. Our design therefore favors outlier-removal mechanisms based on linear relations and simple comparisons, which can be verified efficiently at the cost of avoiding more complex robust aggregation rules.

\section{Conclusions and Future Work}\label{sec:Conclusions}

In this work, we introduce a fully verifiable and robust FL aggregation method that achieves performance comparable to leading non-verifiable defenses. Although verifiability constrains the complexity of robust aggregation, our experiments show that the proposed method withstands a broad range of sophisticated poisoning attacks with only a limited reduction in model accuracy. The additional verification overhead may be acceptable in many cross-silo FL scenarios: at the evaluated scales, proofs can be generated and verified within minutes and require hundreds of megabytes of storage. This trade-off is particularly reasonable in sensitive applications with long data-acquisition and training cycles. For instance, in collaborative medical imaging involving MRI or CT data, data acquisition and curation can take substantially longer than a single aggregation step. In such settings, additional minutes to hours of cryptographic computation may be acceptable in exchange for stronger integrity and auditability guarantees.

Our study also suggests several directions for future work. First, the robustness of the aggregation could be further improved while maintaining full verifiability. As discussed in Section 4, several approaches could be explored, although they may significantly affect performance. In particular, aggregation methods requiring a non-constant number of steps, such as dynamically selecting the number of clusters, remain challenging to verify efficiently. Second, a natural cryptographic direction is to investigate a quantum-safe variant based on post-quantum cryptography and assess its cost. This is arguably not an urgent concern because our solution provides unconditional privacy, even against a large-scale, fault-tolerant quantum computer. Quantum adversaries could instead compromise the binding property of commitments or the soundness of zero-knowledge proofs, but would need access to a quantum computer at the time of aggregation. Thus, store-now-decrypt-later attacks do not apply to our method.

Finally, an interesting extension is a multi-aggregator setting, in which aggregation is performed by multiple aggregators while preserving the same verifiability requirements.

%

\begin{acks}
Funded by the EU CHIST-ERA initiative (PCI2023-145989-2 funded by MICIU/AEI/10.13039/501100011033, Austrian Science Fund (FWF): I6650-N, and by the European Union NextGenerationEU/PRTR). And by the European Union under the European Defence Fund under grant agreement No.~101121403 -- NEWSROOM, under the Horizon Europe Research and Innovation programme under grant agreement No.~101168311 -- LICORICE). Views and opinions expressed are however those of the author(s) only and do not necessarily reflect those of the European Union or the European Commission. Neither the European Union nor the granting authority can be held responsible for them.
\end{acks}

\bibliographystyle{ACM-Reference-Format}
\bibliography{biblio.bib}


\clearpage
\appendix
 \section*{Open Science}
 We are committed to openly sharing the research artifacts associated
with this work to support transparency and reproducibility. To facilitate reuse, we split our artifacts into two repositories: (i) the code used to evaluate robust aggregation (along with the datasets used to train the ML models) is available at \url{https://anonymous.4open.science/r/robust-fl-aggregation-5465};
and (ii) the implementation of the verifiability components, together with the corresponding benchmarks, is available at \url{https://anonymous.4open.science/r/robust-fl-verification-pipeline-33C8}.
%
%

\section{Security proof (\Cref{thm:security})}\label{app:proof_sec}
\begin{proof}
The theorem follows from the following three observations.

\paragraph{Binding to the finalized input set}
Every client submission contains commitments that are authenticated together with the complete session and round context.
The verifier reconstructs the public proof statement from the complete finalized set $\mathcal{S}$.
Therefore, an accepting proof necessarily refers to exactly the authenticated submissions contained in $\mathcal{S}$.

More precisely, the aggregator cannot replace a commitment without breaking the authentication mechanism, and it cannot open an authenticated commitment to a different value without breaking the binding property of the commitment scheme.
Similarly, omitting or adding a submission changes the public statement constructed by the verifier and therefore invalidates the corresponding Fiat-Shamir challenge.
Context binding additionally prevents submissions or proof components from being replayed across sessions or training rounds.

If the input validation extension is enabled, knowledge soundness of $\NIZK^{\mathrm{Norm}^2_i}$ additionally guarantees that every accepted client vector satisfies the coefficient and normalization constraints required by the aggregation procedure.
Without this extension, the same property follows from the stated non-collusion assumption between the submitting client and the aggregator.

\paragraph{Correctness of the verified computation}
Consider an accepting transcript for the proof components $\NIZK^{\mathsf{ip}}$, $\NIZK^{\mathsf{TH}}$, and $\NIZK^{\mathsf{agg}}$.
First, the relations in $\NIZK^{\mathsf{ip}}$ bind the reference-model coordinates, the client-specific inner products, their squares, and the squared norm of the reference model to the authenticated client commitments.
Thus, the committed similarity-related values are exactly those obtained from the finalized client inputs.

Second, the relations in $\NIZK^{\mathsf{TH}}$ guarantee that each commitment $\com^{\THD}_i$ contains the bit $\delta_i \in \{0,1\}$ prescribed by the threshold comparison.
In particular, $\delta_i=1$ if and only if the corresponding encoded similarity satisfies the public inclusion condition, and $\delta_i=0$ otherwise.
The range restrictions and the parameter bounds from \Cref{sec:discussion} ensure that these relations have their intended integer interpretation and cannot be satisfied through modular wraparound.

Finally, the relations in $\NIZK^{\mathsf{agg}}$ bind the output commitments to the aggregation of exactly those client updates for which $\delta_i=1$.
The accompanying cardinality and division proofs guarantee that the released coordinates are the correctly rounded average of the selected updates without revealing the number of selected clients.
For the weighted extension, the corresponding proof relations additionally bind the aggregation coefficients to the prescribed similarity-dependent weights.

It follows that an accepting transcript implies that the released model is the result of applying the publicly specified aggregation procedure to the complete finalized input set $\mathcal{S}$.
Hence, an aggregator that produces an accepting proof for an inconsistent output would either violate knowledge soundness, open a commitment inconsistently, forge an authenticated submission, or violate the binding of the Fiat--Shamir transcript to its public statement.
Under the assumptions of the theorem, each of these events occurs only with negligible probability.

\paragraph{Public-transcript privacy}
The privacy claim concerns public verifiers and clients that do not obtain the openings of other clients' commitments.
The aggregator is outside this privacy boundary because it receives the client updates and commitment openings in plaintext and therefore knows the intermediate values and selection decisions.

Pedersen commitments are perfectly hiding, and thus the publicly posted commitments reveal no information about their committed values.
Moreover, by zero knowledge, each proof component can be simulated from its public statement without access to the corresponding witnesses.
Since the proof components are composed over the same public statement and are bound to the same session and round context, their simulators can be invoked jointly to simulate the complete public proof transcript.

Consequently, the public transcript reveals no information about individual client updates, client-level inclusion decisions, or the cardinality of the selected set beyond what is already implied by the public parameters, the authenticated submission metadata, and the released aggregate.
\end{proof}
 \section{Extensions of the Aggregation Algorithms}
 \label{app:extensions}

\subsection{Dynamic threshold}


As stated previously in the main body of this manuscript, we develop an adaptive threshold that adapts to the distribution of $C$, at the expense of additional verification cost.
To improve later cluster separability, we apply a cubic transformation to these scores.
This mapping compresses values toward zero while preserving sign.
As benign updates typically have cosine similarity close to the mean, as they are more, i.e., close to 1, their scores shrink less than those of misaligned (malicious) updates, which are closer to 0, thereby increasing the contrast between the two groups.
We choose the exponent \(3\) as a compromise: it preserves negativity for \(c_i < 0\), while larger exponents were too aggressive, driving most scores close to zero and reducing useful resolution.

We describe our method in \Cref{alg:threshold_selection}.
We use inter-means classification, i.e., one-dimensional $K$-means~\cite{lloyd1982least} with $K=2$.
The procedure initializes $T_0 = mean(C)$, and then partitions $C$ into two sets: $C_1 = \{c_i\in C | c_i <T_0\}$ and $C_2 = \{c_i\in C | c_i \geq T_0\}$.
The threshold is updated as the midpoint between the two cluster means: $T_{1} = \frac{mean(C_1) + mean(C_2)}{2}$.
We iterate until convergence, i.e., $|T_{r+1} - T_r| < \epsilon$, and select the final value $T_r$ as the dynamic threshold.

\begin{algorithm}
\caption{\textsc{Threshold\_Selection}$(C,\epsilon)$}
\label{alg:threshold_selection}
\begin{algorithmic}[1]
\REQUIRE Similarities vector $C$, convergence threshold $\epsilon$
\ENSURE Threshold $T$
    \STATE $C \gets (c_k^3)_{k\in\numClients}$

    \STATE $T_0 \gets \text{mean}(C)$
    \STATE $r \gets 0$
    \REPEAT
        \STATE $C_1 \gets \{c \in C \mid c < T_r\}$
        \STATE $C_2 \gets \{c \in C \mid c \ge T_r\}$
        \STATE $T_{r+1} \gets \dfrac{\text{mean}(C_1)+\text{mean}(C_2)}{2}$
        \STATE $r \gets r+1$
    \UNTIL{$|T_{r}-T_{r-1}|<\epsilon$}
    \STATE $T \gets T_r$

\RETURN $T$
\end{algorithmic}
\end{algorithm}


\subsection{Final clustering decision}
The inter-means classification described above always partitions $C$ into two clusters.
However, when no malicious clients are present, we would like to retain all clients in order to avoid discarding useful information.
In our experiments, we accept the two-cluster partition.
Nevertheless, to handle the general case, we design a verifiable procedure to decide whether it is appropriate to use $K=2$ clusters or to keep a single cluster ($K=1$), as shown in \Cref{alg:dynamic_clustering}.

We first define the following quantities: $\Delta_i$ denotes the differences between consecutive points, $\Delta_{\max}=\max_i(\Delta_i)$, and $\Delta_{\text{mean}}=\text{mean}(\Delta_i)$.
For cluster $i$, let $\mu_i$ and $\sigma_i$ be its mean and standard deviation, respectively. We then decide between $K\in{1,2}$ using the conditions

\begin{equation}
K=2 \Leftrightarrow
\begin{cases}
\dfrac{\Delta_{\max}}{\Delta_{\text{mean}}}\geq k_1, \\
\dfrac{|\mu_1 - \mu_2|}{\max(\sigma_1,\sigma_2)}\geq k_2,\\
|\mu_1 - \mu_2|\geq 0.02.
\end{cases}
\end{equation}

If all inequalities are accomplished, we set $K=2$.
The constant $0.02$ in the third inequality corresponds to $1\%$ of the total length of the value space.

In practice, $C_1$ may occasionally contain a distant outlier, which inflates $\Delta_{\text{mean}}$ and can cause the first condition to fail even when the two clusters are clearly separated by a pronounced valley.
To address this case, when any of the above inequalities is violated, we make a second condition.
In particular, we strengthen the first and third conditions and accept $K=2$ whenever
$K=2 \Leftrightarrow \left(\dfrac{\Delta_{\max}}{\Delta_{\text{mean}}}\geq k_3 > k_1 \text{ and } |\mu_1 - \mu_2|\geq 0.05\right)$.

Finally, since benign clients are expected to be closer to the mean model update, their similarity scores should be closer to $1$ than those of malicious clients.
Therefore, when $K=2$, we assume that $C_2$ corresponds to the benign cluster.

\begin{algorithm}
\caption{\textsc{Dynamic\_Clustering}$(C,k_1,k_2,k_3)$}
\label{alg:dynamic_clustering}
\begin{algorithmic}[1]
\REQUIRE Similarities vector $C$, parameters $k_1,k_2,k_3$
\ENSURE Number of clusters $K$

\STATE Sort $C$ increasingly to get $(c_1,\dots,c_{|C|})$
\STATE Compute $\Delta_i \gets |c_{i+1}-c_i|$ for $i=1,\dots,|C|-1$
\STATE $\Delta_{\max} \gets \max_i \Delta_i$, \quad $\Delta_{\text{mean}} \gets \text{mean}_i(\Delta_i)$
\STATE Split $C$ into two parts around the largest gap (at $\arg\max_i \Delta_i$), obtaining clusters $C^{(1)}, C^{(2)}$
\STATE Compute $(\mu_1,\sigma_1)$ from $C^{(1)}$ and $(\mu_2,\sigma_2)$ from $C^{(2)}$

\IF{$\dfrac{\Delta_{\max}}{\Delta_{\text{mean}}}\ge k_1 \AND \dfrac{|\mu_1-\mu_2|}{\max(\sigma_1,\sigma_2)}\ge k_2 \AND |\mu_1-\mu_2|\ge 0.02$}
    \STATE $K \gets 2$
\ELSIF{$\dfrac{\Delta_{\max}}{\Delta_{\text{mean}}}\ge k_3 \AND |\mu_1-\mu_2|\ge 0.05$}
    \STATE $K \gets 2$
\ELSE
    \STATE $K \gets 1$
\ENDIF

\RETURN $K$
\end{algorithmic}
\end{algorithm}
\section{Experiment settings}\label{app:settings}


The datasets used in this work are FEMNIST, Synthetic, and Sentiment140. \textbf{FEMNIST} is a federated version of EMNIST~\cite{cohen2017emnist} created by LEAF~\cite{caldas2018leaf} available in their GitHub repository\footnote{\url{https://github.com/TalwalkarLab/leaf}}.
FEMNIST is a dataset of handwritten characters used for image classification.
LEAF takes the EMNIST dataset and divides it into 3550 non-iid clients.
It contains 62 classes (the numbers from 0 to 9, and 52 letters in both upper and lower case).
The \textbf{Synthetic} dataset is a dataset created also by LEAF.
It is designed to simulate statistical heterogeneity across clients.
Each client represents a synthetic task generated from client-specific model parameters and non-IID feature distributions.
In the configuration used in the paper, it contains 1~000 clients, 60 features, and 5 classes, with labels generated artificially as class indices 0–4 \textbf{Sentiment\footnote{\url{https://www.kaggle.com/datasets/kazanova/sentiment140}}} \cite{go2009twitter} is a dataset which contains 1~600~000 tweets extracted using the Twitter API.
The tweets have been annotated (0 = negative, 4 = positive) and they can be used to detect sentiment.

The model used for FEMNIST dataset is a Convolutional Neural Network (CNN) composed of four hidden layers: three 8x8 convolutional layers with 8, 16, and 24 channels respectively, and a fully connected layer of 128 neurons.
All hidden layers use the ReLu activation function, and the output layer has 62 neurons with softmax activation function.
We use Adam as the optimizer with a learning rate of 0.001.
Finally, the batch size is 64.

The model used for Synthetic dataset is a Multi-Layer Perceptron (MLP) composed of two hidden fully connected layers with 128 and 64 neurons, respectively.
All hidden layers use the ReLU activation function, and the output layer has 5 neurons with softmax activation function.
The input layer receives feature vectors of 60 dimensions

Finally, the model used for Sentiment140 dataset is a Long Short-Term Memory network (LSTM) designed for text classification.
It first includes an embedding layer that maps each input token into a dense vector representation, followed by a SpatialDropout1D layer with a dropout rate of 0.4 to reduce overfitting.
Then, an LSTM layer is used with dropout and recurrent dropout rates of 0.2.
Finally, the output layer has a number of neurons equal to the number of classes, with softmax activation function.
We use Adam as the optimizer and sparse categorical cross-entropy as the loss function.

In our case, for FEMNIST, Synthetic, and Sentiment140, we select 50, 20 and 20 random clients that remain fixed throughout the experiments.
The malicious clients are chosen randomly from the same client pool, and are re-sampled in each experiment run.
Additionally, we assume a scenario with no dropouts or \emph{stragglers}~\cite{park2021sageflow}, meaning the nodes remain active throughout the entire training process.


\section{Poisoning attacks considered}\label{app:poison}

The attacks considered in this work are selected as representative and influential poisoning strategies.
We study one canonical data poisoning method and several local model poisoning attacks, covering both simple and more advanced adversarial settings.
For data poisoning, we use the label-flipping attack, a common dirty-label approach where the adversary only alters labels and need not control the training pipeline.
For model poisoning, we focus on several untargeted attacks, as these have been reported to be harder to detect than targeted ones~\cite{kiourti2020trojdrl}.

\begin{enumerate}
	\item \textbf{Label Flipping: }Label flipping~\cite{lyu2020threats} is a data poisoning technique in which a client’s labels are systematically reassigned from one (source) class to another (target) class, while the underlying training samples remain unchanged. In our case we change the class randomly.

	\item \textbf{LIE: }LIE (Little Is Enough)~\cite{baruch2019little} is a model poisoning attack that adds small perturbations to the average model update.
		Malicious clients compute the mean ($\nabla^\numMal$) and standard deviation ($\sigma^\numMal$) of their weights after each round, then form the poisoned update as $\nabla^\numMal + z\sigma^\numMal$, where $z$ a user-chosen constant.
		This keeps the update within typical detection thresholds. In our experiments $z=1.5$. 

	\item \textbf{STATOPT: }STATOPT (Static Optimization)~\cite{fang2020local} first computes the average weights of the malicious clients ($\nabla^\numMal$) and defines a static adversarial direction $\omega = -sign(\nabla^\numMal)$.
		The malicious update transmitted to the server is $-\gamma\omega$, where $\gamma$ is a scalar selected by the attackers. In our experiments $\gamma=1$. 

	\item \textbf{Mimic: }The Mimic attack~\cite{karimireddy2020byzantine} selects a benign client whose update exhibits high variance by the formula \[\arg\max_{\lVert z\rVert = 1}\; 
z^{\top}\!\left(\sum_{t\in I_0}\sum_{i\in G}(x_i^{t}-\mu)(x_i^{t}-\mu)^{\top}\right) z,\\[4pt]\] and makes all malicious clients copy that update.
		This coordinated replication biases the global aggregation.
		Since the update originally comes from a benign client, the resulting attack becomes more challenging to detect.

	\item \textbf{Min-max: }The min-max attack, introduced in~\cite{shejwalkar2021manipulating}, aims to maximize the separation between benign and malicious updates while still bypassing the aggregation rule’s defenses.
		Malicious clients first compute the mean of their weights ($\nabla^\numMal$), then determine a perturbation direction $\theta$ and scale it by a factor $\gamma$, yielding the final adversarial update $\nabla^\numMal + \gamma\theta$.
		Here, $\theta$ is an adversarially chosen direction in the model update space, and $\gamma$ is obtained by solving an optimization problem that keeps the resulting update within the statistical limits defined by the other malicious clients. In our experiments, $\theta$ is the standard deviation of the weights.

	\item \textbf{Min-sum: }The min-sum attack, also proposed in~\cite{shejwalkar2021manipulating}, closely resembles min-max in that it uses the same key quantities: the average $\nabla^\numMal$, the direction $\theta$, and the scaling factor $\gamma$.
		However, in this case the optimization objective is to maximize the sum of distances between the perturbed update $\nabla^\numMal + \gamma\theta$ and all other malicious updates, subject to a constraint on their overall dispersion. In our experiments, $\theta$ is the standard deviation of the weights.
\end{enumerate}

\section{Results on Synthetic and Sentiment140 datasets}\label{app:restresults}

In this appendix we report the results of the experiments on the two
remaining datasets, namely the \emph{Synthetic} benchmark and
\emph{Sentiment140}, which complement the FEMNIST results discussed in the
main text. For each dataset we plot detection accuracy and final accuracy
under the same six poisoning attacks (Label Flipping, LIE, STATOPT, Mimic,
Min-Max, and Min-Sum) and the same range of malicious-client fractions
($10\%$--$40\%$). The attack-free FedAvg accuracy is reported as a dashed
reference line ($0.92$ for Synthetic and $0.80$ for Sentiment140). For
context, we also include FLAegis (our prior, non-ZK-friendly defense) as an
upper-bound reference.

\paragraph{Synthetic.} \Cref{fig:SYNTHETIC_comparison_det} shows that our
detector identifies essentially all malicious clients across the full range
of attacks and ratios, with detection rates at or near $1.0$ on LIE,
STATOPT, Mimic, and Label
Flipping, and remaining high ($\ge 0.8$ in mean) Min-Sum. On the Synthetic dataset, our method struggles more under the Min-max attack. FLAegis exhibits the same near-perfect
behavior, as expected. The remaining detector-based baselines are each
strong on a subset of attacks but brittle on the others: Lomar is essentially
perfect under LIE, Mimic, Min-Max, and Min-Sum but degrades monotonically
under STATOPT (down to $\approx 0.60$ at $40\%$) and is consistently weaker
than ours under Label Flipping; M3D-FL is the most uniform baseline but
never the best, declining steadily with the malicious fraction (e.g., from
$\approx 0.90$ at $10\%$ to $\approx 0.60$ at $40\%$ on Mimic and Min-Sum);
SafeFL is acceptable only on Label Flipping and collapses elsewhere, falling
below $0.3$ under Min-Max and Min-Sum at $40\%$.

The corresponding final accuracies in \Cref{fig:SYNTHETIC_comparison_acc}
mirror the results of FEMNIST. Our method stays
close to the ideal reference ($\approx 0.90$ across all attacks and ratios,
with a single mild dip to $\approx 0.8$ at $40\%$ under STATOPT and $0.85$ in Min-sum), and is
indistinguishable from FLAegis except in this last setting. The non-ZK
detector-based baselines exhibit the trade-off we already observed in the
main text: Lomar is near-ideal on most attacks but drops to $\approx 0.22$
under STATOPT at $40\%$; M3D-FL collapses under LIE (down to $\approx 0.38$)
and under STATOPT (down to $\approx 0.67$); SafeFL fails under LIE
($\approx 0.37$ at $40\%$) and STATOPT (with very large variance at $10\%$);
and Mimic at $20$--$40\%$ degrades M3D-FL to $\approx 0.65$--$0.76$. The
classical robust aggregators Median and Krum are the most stable baselines
and remain around $0.87$--$0.90$, although they degrade monotonically with
the malicious fraction and, except in a few isolated configurations, stay
below our accuracy.

\paragraph{Sentiment140.} The detection results in
\Cref{fig:sentiment140_comparison_det} are qualitatively consistent with
those on Synthetic. Our method achieves detection rates of $\approx 0.95$
on Label Flipping, $\approx 1.0$ on LIE and Mimic, $\ge 0.8$ on STATOPT,
and $\approx 0.85$--$0.97$ on Min-Max and Min-Sum, with low variance across
runs. FLAegis behaves almost identically. Lomar attains perfect detection
under LIE, Mimic, Min-Max, and Min-Sum, but---as on Synthetic---degrades
sharply under STATOPT, falling below $0.60$ at $40\%$. M3D-FL again provides
uniform but never best-in-class detection, and SafeFL is the weakest
baseline, dropping below $0.7$ on most attacks at low malicious fractions
and remaining the most variable defense overall.

The final-accuracy plot in \Cref{fig:sentiment140_comparison_acc} confirms
the same trend. Our method tracks FLAegis and the ideal reference closely
on Label Flipping, LIE, Mimic, Min-Max, and Min-Sum (mean accuracy
$0.76$--$0.79$ across all ratios). The only regime in which our method is
clearly outperformed is STATOPT, where its accuracy decreases with the
malicious fraction to $\approx 0.75$ at $40\%$ while SafeFL, Median, and
Krum remain near $0.79$; FLAegis also remains near $0.79$ in this setting,
confirming that the gap is a property of the additional ZK-friendliness
constraint rather than of the underlying detection idea. On the remaining
attacks the picture matches the main results: Lomar collapses under STATOPT
(down to $\approx 0.51$ at $40\%$), M3D-FL is consistently weaker on LIE,
Mimic, Min-Max, and Min-Sum (around $0.63$--$0.68$), and SafeFL fails on
LIE ($\approx 0.57$ at $40\%$); Median and Krum stay stable but slightly
below our accuracy.

\begin{figure*}[!htbp]
	\centering
	\includegraphics[width=\linewidth]{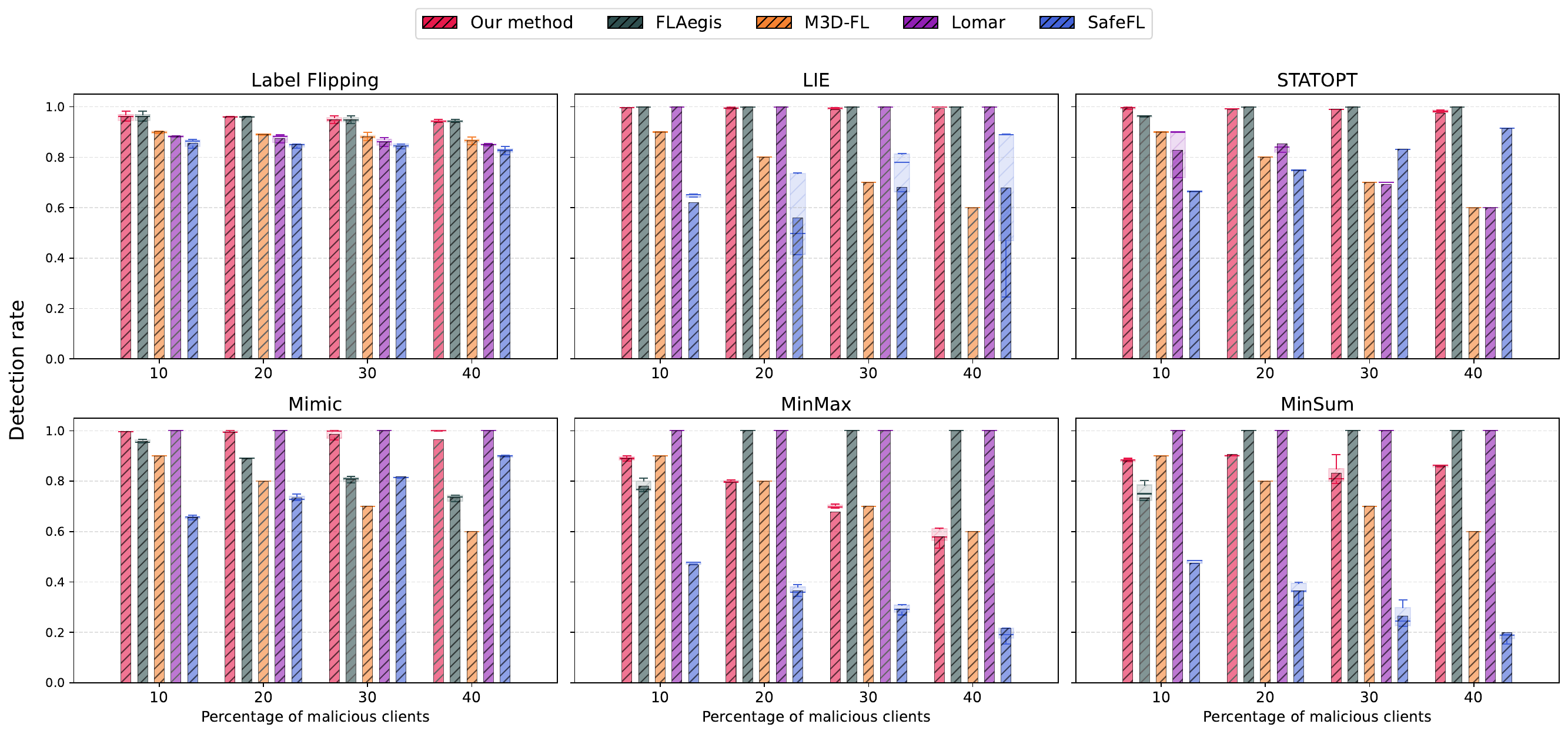}
	\caption{Detection accuracy of the different methods against the different poisoning attacks with different percentages with the synthetic dataset.}
	\label{fig:SYNTHETIC_comparison_det}
\end{figure*}

\begin{figure*}[!htbp]
	\centering
	\includegraphics[width=\linewidth]{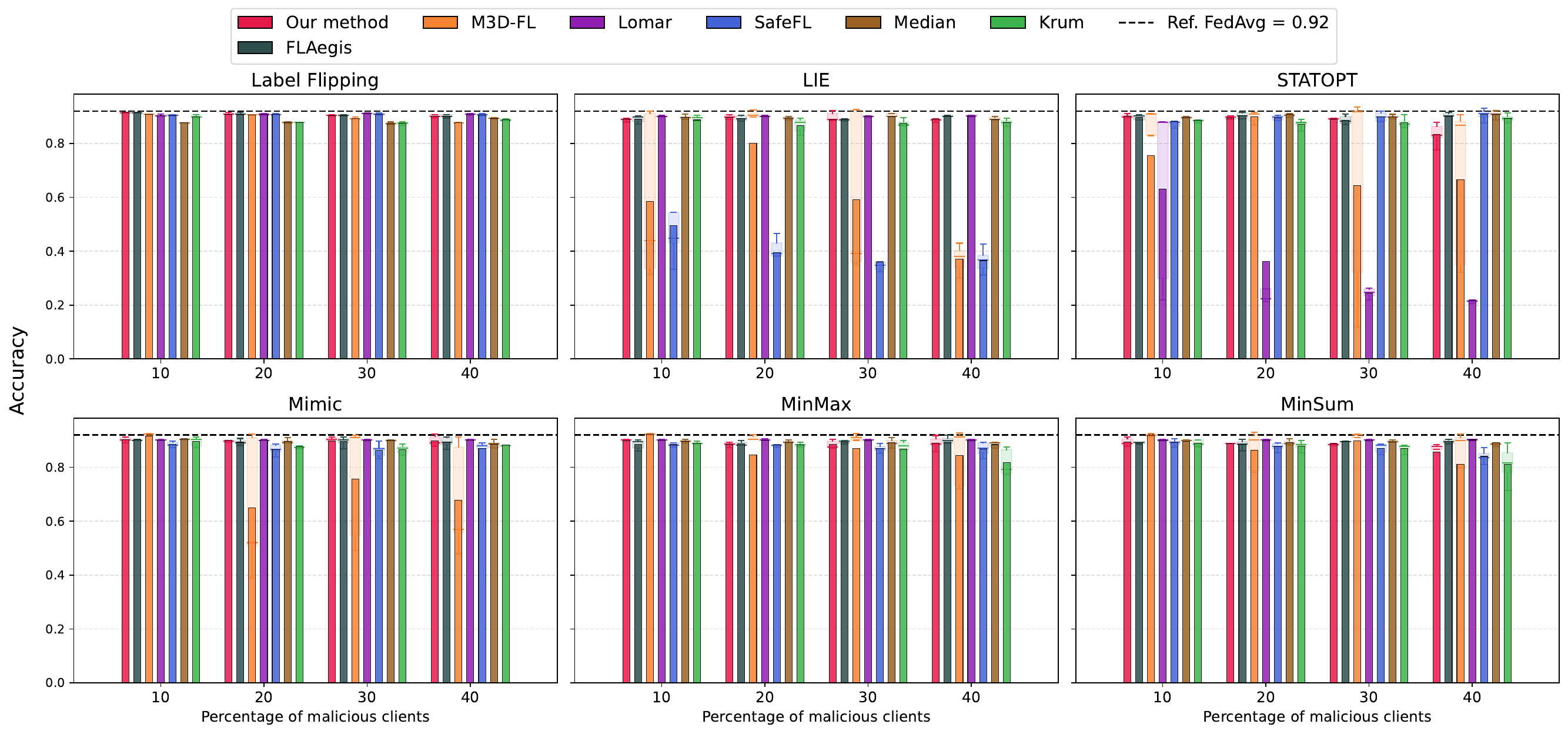}
	\caption{Final accuracy of the different methods against the different poisoning attacks with different percentages with the synthetic dataset.}
	\label{fig:SYNTHETIC_comparison_acc}
\end{figure*}

\begin{figure*}[!htbp]
	\centering
	\includegraphics[width=\linewidth]{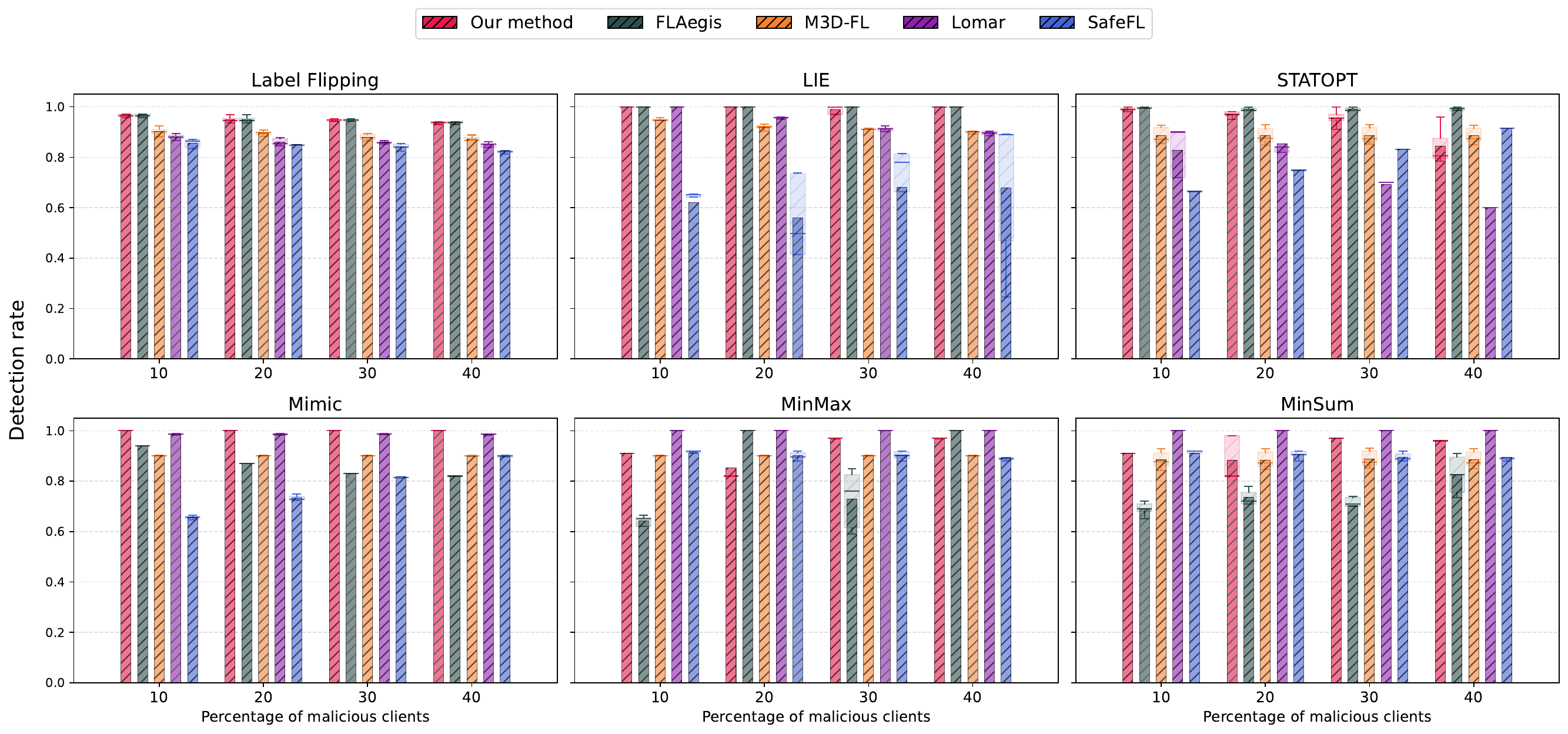}
	\caption{Detection accuracy of the different methods against the different poisoning attacks with different percentages with the sentiment140 dataset.}
	\label{fig:sentiment140_comparison_det}
\end{figure*}

\begin{figure*}[!htbp]
	\centering
	\includegraphics[width=\linewidth]{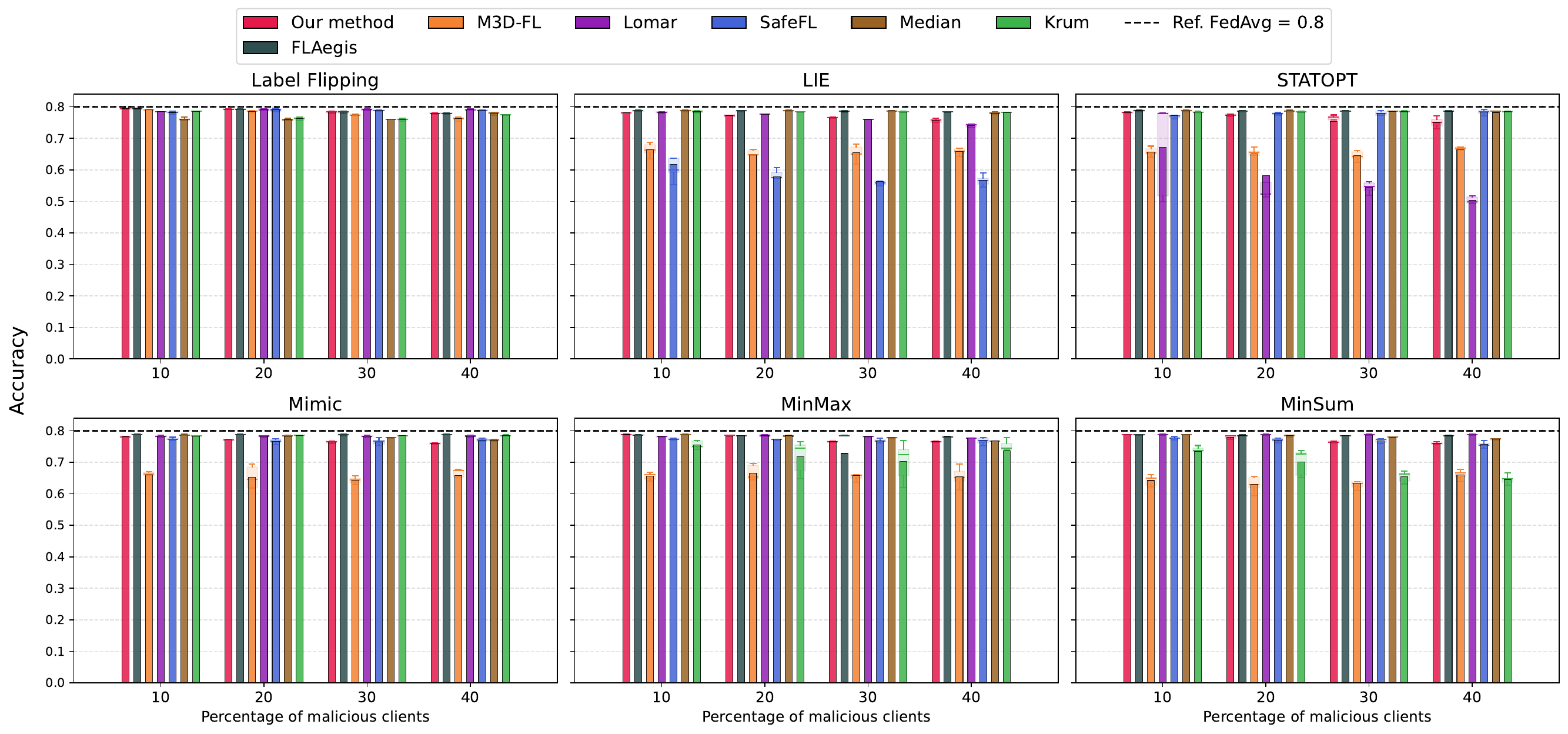}
	\caption{Final accuracy of the different methods against the different poisoning attacks with different percentages with the sentiment140 dataset.}
	\label{fig:sentiment140_comparison_acc}
\end{figure*}

\paragraph{Summary.} Across both additional datasets, the same picture
established in the main text holds: every non-verifiable baseline that
edges out our method on a particular attack collapses on at least one
other, while our method withstands all attacks without catastrophic failure
and tracks FLAegis closely in almost every configuration. These appendix
results therefore reinforce, on two further datasets, the claim that our
ZK-friendly defense achieves the most uniform robustness although we have to implement verifiable functions. Additionally, we show in Table \ref{tab:comparison_app} an extension of \Cref{tab:comparison} with all the datasets expanded.

\begin{table}[]
\centering
\begin{tabular}{llcc}
\hline
 & \multicolumn{3}{c}{\textbf{AD}} \\ \hline
\textbf{Method} & \textbf{Dataset} & \textbf{Mean} & \textbf{Standard deviation} \\ \hline
\multirow{4}{*}{Our method} & FEMNIST & 0.0188 & 0.0136 \\
 & sentiment140 & 0.0299 & 0.0040 \\
 & Synthetic & 0.0315 & 0.0081 \\ \cline{2-4} 
 & Average & 0.0267 & 0.0086 \\ \hline
\multirow{4}{*}{FLAegis} & FEMNIST & 0.0115 & 0.0021 \\
 & sentiment140 & 0.0171 & 0.0074 \\
 & Synthetic & 0.0261 & 0.0030 \\ \cline{2-4} 
 & Average & 0.0182 & 0.0042 \\ \hline
\multirow{4}{*}{M3D-FL} & FEMNIST & 0.2824 & 0.2441 \\
 & sentiment140 & 0.1473 & 0.0065 \\
 & Synthetic & 0.1559 & 0.1178 \\ \cline{2-4} 
 & Average & 0.1952 & 0.1228 \\ \hline
\multirow{4}{*}{Lomar} & FEMNIST & 0.1244 & 0.2296 \\
 & sentiment140 & 0.0618 & 0.0914 \\
 & Synthetic & 0.1262 & 0.2404 \\ \cline{2-4} 
 & Average & 0.1041 & 0.1871 \\ \hline
\multirow{4}{*}{SafeFL} & FEMNIST & 0.4527 & 0.3584 \\
 & sentiment140 & 0.0667 & 0.0850 \\
 & Synthetic & 0.1368 & 0.2125 \\ \cline{2-4} 
 & Average & 0.2187 & 0.2186 \\ \hline
\multirow{4}{*}{Median} & FEMNIST & 0.0393 & 0.0104 \\
 & sentiment140 & 0.0213 & 0.0033 \\
 & Synthetic & 0.0241 & 0.0044 \\ \cline{2-4} 
 & Average & 0.0282 & 0.0060 \\ \hline
\multirow{4}{*}{Krum} & FEMNIST & 0.0303 & 0.0001 \\
 & sentiment140 & 0.0473 & 0.0457 \\
 & Synthetic & 0.0467 & 0.0090 \\ \cline{2-4} 
 & Average & 0.0414 & 0.0183 \\ \hline
\end{tabular}
\caption{Summary table of the AD results for each dataset and method.}
\label{tab:comparison_app}
\end{table}

\section{Cost Estimations}\label{s:cost_estimates}

In \Cref{tab:estimation_protocol}, we provide an estimate of the number of multi-exponentiations (MExp) required for proof generation and verification for both the main protocol and each extension discussed in \Cref{sec:robust_verif_fl}.
The proof generation cost includes the multi-exponentiations performed during the Schnorr proof generation and simulation phases, as well as the computation of Pedersen commitments.
For verification, we assume that the verifier employs a batch verification mechanism, whereby multiple proofs are verified simultaneously through a single aggregated verification equation.
To prevent a malicious prover from constructing correlated invalid proofs that cancel each other out during aggregation, the verifier samples a random coefficient vector A and applies it to each proof instance in the aggregated verification equation.
Batch verification significantly reduces the computational cost of verification.
However, this optimization sacrifices perfect soundness, i.e., it introduces a negligible, but non-zero, probability that a malicious prover succeeds in passing verification, which for the selected parameters is however negligible.

\begin{table*}[!t]
	\begin{center}
\caption{Estimated operation costs for proof generation and verification of the protocol and its extensions, expressed in terms of multi-exponentiations. We denote by $\mathrm{ME}_n$ a multi-exponentiation involving $n$ generators of $\GG$.}
\label{tab:estimation_protocol}
\begin{tabular}{l|p{9cm}|p{5cm}}
\toprule
 & Proof Generation & Verification \\
\midrule
Protocol & $ (\numClients+1)$ME$_{\dimModel + 1} + \dimModel $ME$_{\numClients + 1} + (3\numClients+3\dimModel) $ME$_3 + \bigl(9\numClients(2\precision + \log_2(\dimModel \numClients)) + 19\numClients + 2\dimModel \log_2(\numClients) + 7\dimModel + \log_2(\numClients) + 1\bigr) $ME$_2 + \bigl(3\numClients(2\precision + \log_2(\dimModel \numClients) + 1) + (2\dimModel+1)(\log_2(\numClients) + 1)\bigr)$ME$_1$ & $ \numClients(2\dimModel + 24(2\precision + \log_2(\numClients) + \log_2(\dimModel)) + 58) + \dimModel\bigl(10\log_2(\numClients) + 17\bigr) + 5\log_2(\numClients) + 5$\\
Extension 1 & ME$_{\dimModel + 1}$ + $(6\precision + \dimModel + 3)$ME$_2$ +  $2(3\precision +2)$ME$_1$ per client & $\numClients(4\dimModel + 30\precision + 13)$  \\
Extension 2 & $3\numClients(2\precision + 2\log_2(\numClients) + \log_2(\dimModel) + 1)$ME$_2$ + $\numClients(2\precision + 2\log_2(\numClients) + \log_2(\dimModel) +1)$ME$_1$ & $11\numClients(2\precision + 2\log_2(\numClients) + \log_2(\dimModel) + 1)$ \\
Extension 3 &  $2\numClients(2\precision +1)($ME$_1$ + ME$_2$)& $\numClients(20\precision +6)$ \\

\bottomrule
\end{tabular}
\end{center}
\end{table*}

The \Cref{tab:estimation_protocol2} gives an estimated execution time for four sets of parameters.
We have obtained those results with estimator \texttt{zka.lc}~\cite{zkalc}.
We ran the estimation with the following environment:  sec256k1 elliptic curve, gnark-crypto library and AWS EC2 m6i.8xlarge machine.

\begin{table*}[!t]
\begin{center}
\caption{Estimated execution time (in seconds) for proof generation and verification of the protocol and its extensions.}
\label{tab:estimation_protocol2}

\begin{tabular}{l|cc|cc|cc|cc}
\toprule
& \multicolumn{2}{c|}{$\dimModel = 10^5, \numClients = 10^2$}
& \multicolumn{2}{c|}{$\dimModel = 10^6, \numClients = 10^2$}
& \multicolumn{2}{c|}{$\dimModel = 10^5, \numClients = 10^3$}
& \multicolumn{2}{c}{$\dimModel = 10^6, \numClients = 10^3$} \\
\cmidrule(lr){2-3} \cmidrule(lr){4-5} \cmidrule(lr){6-7} \cmidrule(lr){8-9}
& Proof & Ver. & Proof & Ver. & Proof & Ver. & Proof & Ver. \\
\midrule

Protocol
& 107.7 & 8.71
& 147.00 & 86.00
& 358.66 & 84.00
& 1803.09 & 842.00 \\

Extension 1
& 15.64 & 16.85
& 155.89 & 168.21
& 15.64 & 168.21
& 155.89 & 1682.00 \\

Extension 2
& 4.71 & 0.048
& 4.92 & 0.050
& 51.2 & 0.408
& 53.2 & 0.422 \\

Extension 3
& 2.84 & 0.026
& 2.84 & 0.026
& 28.39 & 0.224
& 28.39 & 0.224 \\

\bottomrule
\end{tabular}
\end{center}
\end{table*}

\end{document}